\documentclass{article}
\usepackage{microtype}
\usepackage{graphicx}
\usepackage{subfigure}
\usepackage{enumitem}
\usepackage{booktabs}
\usepackage{hyperref}

\newcommand{\argmin}{\mathop{\mathrm{argmin}}}
\newcommand{\argmax}{\mathop{\mathrm{argmax}}}

\def\R{\mathbb{R}}

\def\E{\mathbb{E}}
\def\P{\mathbb{P}}

\def\cA{\mathcal{A}}

\def\cD{\mathcal{D}}

\def\cM{\mathcal{M}}

\def\l{\left}
\def\r{\right}

\def\eps{\varepsilon}

\def\event{\mathcal{E}}

\def\ball{\mathbb{B}}
\def\normal{\mathcal{N}}

\def\hyperplane{\mathcal{H}}

\def\empR{\widehat r}
\def\noisevec{\eta}
\newcommand{\norm}[2]{\l\|#2\r\|_{#1}}

\def\improve{\Delta}
\newcommand{\trustRegion}[1]{\mathsf{C}\l(#1\r)}
\newcommand{\empImprove}[1]{\widehat \Delta_{#1}}
\newcommand{\optImprove}[1]{\Delta^*_{#1}}

\def\empEps{\widehat \eps_*}

\def\selectImprove{\widehat \Delta_*}

\newcommand{\diag}[1]{\operatorname{diag}\l(#1\r)}

\def\lcb{\mathsf{LCB}}

\def\OfflineDataset{\mathcal{D}}

\def\eps{\varepsilon}
\newcommand{\ceil}[1]{\l\lceil #1 \r\rceil}
\newcommand{\floor}[1]{\l\lfloor #1 \r\rfloor}

\newcommand{\GaussSupremum}[1]{\mathcal{G}\l(#1\r)}
\newcommand{\tTrustRegion}[1]{\widetilde{\mathsf{C}}\l(#1\r)}
\def\teta{\widetilde{\eta}}
\def\timprove{\widetilde{\improve}}
\def\mdp{\cM}
\newcommand{\bernoulli}[1]{\mathsf{Bernoulli}\l(#1\r)}
\usepackage[accepted]{icml2023}
\usepackage{amsmath}
\usepackage{amssymb}
\usepackage{mathtools}
\usepackage{amsthm}
\usepackage[capitalize,noabbrev]{cleveref}

\theoremstyle{plain}
\newtheorem{theorem}{Theorem}[section]
\newtheorem{proposition}[theorem]{Proposition}
\newtheorem{lemma}[theorem]{Lemma}
\newtheorem{corollary}[theorem]{Corollary}
\theoremstyle{definition}
\newtheorem{definition}[theorem]{Definition}

\theoremstyle{remark}

\usepackage{multirow}
\usepackage{array}
\usepackage{wrapfig}
\newcolumntype{C}[1]{>{\centering\arraybackslash}p{#1}}

\usepackage[textsize=tiny]{todonotes}

\icmltitlerunning{Is Offline Decision Making Possible with Only Few Samples?}

\begin{document}

\twocolumn[
\icmltitle{Is Offline Decision Making Possible with Only Few Samples? \\
Reliable Decisions in Data-Starved Bandits via Trust Region Enhancement}

\icmlsetsymbol{equal}{*}

\begin{icmlauthorlist}
\icmlauthor{Ruiqi Zhang}{stats}
\icmlauthor{Yuexiang Zhai}{EECS}
\icmlauthor{Andrea Zanette}{EECS,ECE}
\end{icmlauthorlist}

\icmlaffiliation{stats}{Department of Statistics, University of California, Berkeley, CA, USA.}
\icmlaffiliation{EECS}{Department of Electrical Engineering and Computer Sciences, University of California, Berkeley, CA, USA}
\icmlaffiliation{ECE}{Department of Electrical and Computer Engineering, Carnegie Mellon University, Pittsburgh, PA, USA}

\icmlkeywords{Multi-armed bandit, Offline reinforcement learning, Trust region.}

\vskip 0.3in
]

\printAffiliationsAndNotice{}

\begin{abstract}
What can an agent learn in a stochastic Multi-Armed Bandit (MAB) problem from a dataset that contains just a single sample for each arm?
Surprisingly, in this work, we demonstrate that even in such a data-starved setting it may still be possible to find a policy competitive with the optimal one. 
This paves the way to reliable decision-making in settings where critical decisions must be made by relying only on a handful of samples.

Our analysis reveals that \emph{stochastic policies can be substantially better} than deterministic ones for offline decision-making.
Focusing on offline multi-armed bandits, we design an algorithm called Trust Region of Uncertainty for Stochastic policy enhancemenT (TRUST) which is quite different from the predominant value-based lower confidence bound approach.
Its design is enabled by localization laws, critical radii, and relative pessimism. 
We prove that its sample complexity is comparable to that of LCB on minimax problems while being substantially lower on problems with very few samples.

Finally, we consider an application to offline reinforcement learning in the special case where the logging policies are known. 
\end{abstract}

\section{Introduction}
In several important problems, critical decisions must be made with just very few samples of pre-collected experience.
For example, collecting samples in robotic manipulation may be slow and costly, and the ability to learn from very few interactions is highly desirable~\citep{hester2013texplore,liu2021deep}. 
Likewise, in clinical trials and in personalized medical decisions, reliable decisions must be made by relying on very small datasets~\citep{liu2017deep}. 
Sample efficiency is also key in personalized education~\citep{bassen2020reinforcement, ruan2023reinforcement}.

However, to achieve good performance, the state-of-the-art algorithms may require millions of samples~\citep{fu2020d4rl}.
These empirical findings seem to be supported by the existing theories: the sample complexity bounds, even minimax optimal ones, can be large in practice due to the large constants and the warmup factors~\citep{menard2021fast, li2022settling, azar2017minimax, zanette2019b}.

In this work, we study whether it is possible to make reliable decisions with only a few samples.
We focus on an offline Multi-Armed Bandit (MAB) problem, which is a foundation model for decision-making~\citep{lattimore2020bandit}.
In online MAB, an agent repeatedly chooses an arm from a set of arms, each providing a stochastic reward. Offline MAB is a variant where the agent cannot interact with the environment to gather new information and instead, it must make decisions based on a pre-collected dataset without playing additional exploratory actions, aiming at identifying the arm with the highest expected reward~\citep{audibert2010best,garivier2016optimal, russo2016simple, ameko2020offline}.
 
The standard approach to the problem is the Lower Confidence Bound (LCB) algorithm~\citep{rashidinejad2021bridging}, a pessimistic variant of UCB~\citep{auer2002finite} that involves selecting the arm with the highest lower bound on its performance. LCB encodes a principle called \emph{pessimism under uncertainty}, which is the foundation principle for most algorithms for offline bandits and
reinforcement learning (RL)~\citep{jin2020pessimism,zanette2020provably,xie2021bellman,yin2021towards, kumar2020conservative, kostrikov2021offline}.

Unfortunately, the available methods that implement the principle of pessimism under uncertainty can fail in a data-starved regime because they rely on confidence intervals that are too loose when just a few samples are available. For example, even on a simple MAB instance with ten thousand arms, the best-known~\citep{rashidinejad2021bridging} performance bound for the LCB algorithm requires 24 samples per arm in order to provide meaningful guarantees, see \cref{sec:one.sample}.
In more complex situations, such as in the sequential setting with function approximation, such a problem can become more severe due to the higher metric entropy of the function approximation class and the compounding of errors through time steps. 

These considerations suggest that there is a ``barrier of entry'' to decision-making, both theoretically and practically:
one needs to have a substantial number of samples in order to make reliable decisions even for settings as simple as offline MAB where the guarantees are tighter.
Given the above technical reasons, 
and the lack of good algorithms and guarantees for data-starved decision problems, it is unclear whether it is even possible
 to find good decision rules with just a handful of samples.

In this paper, we make a substantial contribution towards lowering such barriers of entry.
 We discover that a carefully-designed algorithm tied 
to an advanced statistical analysis can 
substantially improve the sample complexity, both theoretically and practically, and enable reliable decision-making with just a handful of samples.
More precisely, we focus on the offline MAB setting where 
we show that even if the dataset contains 
just a \emph{single sample} in every arm, 
it may still be possible to compete with the optimal policy.
This is remarkable, because with 
just one sample per arm---for example from a Bernoulli distribution---it is impossible to estimate the expected payoff of any of the arms! Our discovery is enabled by several key insights:
\begin{itemize}
	\item We search over \emph{stochastic} policies, which we show can yield better performance for offline-decision making;
	\item We use a \emph{localized} notion of metric entropy to carefully control the size of the stochastic policy class that we search over; 
	\item We implement a concept called \emph{relative pessimism} to obtain sharper guarantees.
\end{itemize}
These considerations lead us to design a trust region policy optimization algorithm called Trust Region of Uncertainty for Stochastic policy enhancemenT (TRUST), one that offers superior theoretical as well as empirical performance compared to LCB in a data-scarce situation. 

Moreover, we apply the algorithm to selected reinforcement learning problems from~\citep{fu2020d4rl} in the special case where information about the logging policies is available.
We do so by a simple reduction from reinforcement learning to bandits, by mapping policies and returns in the former to actions and rewards in the latter, thereby disregarding the sequential aspect of the problem.
Although we rely on the information of the logging policies being available, the empirical study shows that our algorithm compares well with a strong deep reinforcement learning baseline (i.e, CQL from~\citet{kumar2020conservative}), without being sensitive to partial observability, sparse rewards, and hyper-parameters.

\section{Additional related work}

Multi-armed bandit (MAB) is a classical decision-making framework~\citep{lattimore2020bandit,lai1985asymptotically,lai1987adaptive,langford2007epoch, auer2002using,bubeck2012regret, audibert2009minimax,degenne2016anytime}. The natural approach in offline MABs is the LCB algorithm~\citep{ameko2020offline,si2020distributionally}, an offline variant of the classical UCB method~\citep{auer2002finite} which is minimax optimal~\citep{rashidinejad2021bridging}. 

The optimization over stochastic policies is also considered in combinatorial multi-armed bandits (CMAB)~\citep{combes2015combinatorial}. 
Most works on CMAB focus on variants of the UCB algorithm~\citep{kveton2015tight,combes2015combinatorial,chen2016combinatorial} or of Thompson sampling ~\citep{wang2018thompson,liu2023variable}, and they are generally online.

Our framework can also be applied to offline reinforcement learning (RL)~\citep{sutton2018reinforcement} whenever the logging policies are accessible. There exist a lot of practical algorithms for offline RL~\citep{fujimoto2019off, peng2019advantage, wu2019behavior, kumar2020conservative, kostrikov2021offline}. Theory has also been investigated extensively in tabular domain and function approximation setting~\citep{nachum2019algaedice, xie2020batch, zanette2021provable,xie2021bellman, yin2022near, xiong2022nearly}. 
Some works also tried to establish general guarantees for deep RL algorithms via sophisticated statistical tools, such as bootstrapping~\citep{thomas2015high,nakkiran2020deep, hao2021bootstrapping,wang2022bootstrapped,zhang2022off}.

We rely on the notion of pessimism, which is a key concept in offline bandits and RL. While most prior works focused on the so-called absolute pessimism~\citep{jin2020pessimism,xie2021bellman,yin2022near,rashidinejad2021bridging,li2023reinforcement}, the work of \citet{cheng2022adversarially} applied pessimism not on the policy value but on the difference (or improvement) between policies. However, their framework is very different from ours.

We make extensive use of two key concepts, namely localization laws and critical radii~\citep{wainwright2019high}, which control the relative scale of the signal and uncertainty. The idea of localization plays a critical role in the theory of empirical process~\citep{geer2000empirical} and statistical learning theory~\citep{koltchinskii2001rademacher,koltchinskii2006local,bartlett2002rademacher,bartlett2005local}. The concept of critical radius or critical inequality is used in non-parametric regression~\citep{wainwright2019high} and in off-policy evaluation~\citep{duan2021optimal,duan2022policy,duan2023finite,mou2022off}.

\section{Data-Starved Multi-Armed Bandits}

In this section, we describe the MAB setting and
give an example of a ``data-starved'' MAB instance where prior methods (such as LCB) can fail.
We informally say that an offline MAB is ``data-starved'' if its dataset contains only very few samples in each arm.

\textbf{Notation}
We let $[n] = \{1,2,...,n\}$ for a positive integer $n.$ 
We let $\norm{2}{\cdot}$ denote the Euclidean norm for vectors and the operator norm for matrices.
We hide constants and logarithmic factors in the $\widetilde{O}(\cdot)$ notation. We let $\ball_p^d(s) = \{x \in \R^d: \l\|x\r\|_p \leq s \}$ for any $s \geq 0 $ and $p \geq 1.$
$a \lesssim b$ ($a \gtrsim b$) means $a \leq C b$ ($a \geq Cb$) for some numerical constant $C.$ $a \simeq b$ means that both $a \lesssim b$ and $b \lesssim a$ hold.

\subsection{Multi-armed bandits}\label{sec:setup.mab}
We consider the case where there are $d$ arms in a set $\cA = \{a_1,...,a_d\}$ with expected reward $r(a_i), i \in [d].$ 
We assume access to an offline dataset $\cD = \l\{(x_i,r_i)\r\}_{i \in [N]}$ of action-reward tuples, where the experienced actions $\l\{x_i\r\}_{i \in [N]}$ are i.i.d. from a distribution $\mu$. 
Each experienced reward is a random variable with expectation $\E [r_i] = r(x_i)$ and independent Gaussian noises $\zeta_i := r(x_i) - \E[r_i].$ For $i \in [d],$ we denote the number of pulls to arm $a_i$ in $\cD$ by $N(a_i)$ or $N_i,$ while the variance of the noise for arm $a_i$ is denoted by $\sigma_i^2.$
We denote the optimal arm as $a^* \in \mathop{\arg\max}_{a \in \cA} [r(a)]$ and the single policy concentrability as $C^* = 1/\mu(a^*)$ where $\mu$ is the distribution
that generated the dataset.
 Without loss of generality, we assume the optimal arm is unique.
We also write $r = (r_1,r_2,...,r_d)^\top.$ Without loss of generality, we assume there is at least one sample for each arm (such arm can otherwise be removed). 

\subsection{Lower confidence bound algorithm}
One simple but effective method for the offline MAB problem is the Lower Confidence Bound (LCB) algorithm, which 
is inspired by its online counterpart (UCB)~\citep{auer2002finite}. 
Like UCB, LCB computes the empirical mean 
$\widehat{r}_i$ associated to the reward of each arm $i$
along with its half confidence width $b_i$. 
They are defined as
\begin{equation}\label{eqn.emp.reward.mab}
    \widehat{r}_i := \frac{1}{N(a_i)} \sum_{k:x_k = a_i} x_k, \
    b_i := \sqrt{\frac{2 \sigma_i^2}{N(a_i)} \log\l(\frac{2d}{\delta}\r)}.
\end{equation}
This definition ensures that each confidence interval brackets the corresponding expected reward with probability $1-\delta$:
\begin{equation}\label{eqn.confidence.bound.LCB}
    \widehat{r}_i - b_i \leq r\l(a_i\r)
    \leq \widehat{r}_i + b_i \quad \forall i \in [d].
\end{equation}
The width of the confidence level depends on the noise level $\sigma_i$, which can be exploited by variance-aware methods~\citep{zhang2021improved, min2021variance, yin2022near, dai2022variance}. When the true noise level is not accessible, we can replace it with the empirical standard deviation or with a high-probability upper bound. For example, when the reward for each arm is restricted to be within $[0,1],$ a simpler upper bound is $\sigma_i^2 \leq 1/4.$ 

Unlike UCB, the half-width of the confidence intervals for LCB is not added, but subtracted, from the empirical mean,
resulting in the lower bound $l_i = \widehat r_i - b_i$.
The action identified by LCB is then the one that maximizes the resulting lower bound,
thereby incorporating the principle of pessimism under uncertainty
\citep{jin2020pessimism,kumar2020conservative}. Specifically, given the dataset $\mathcal{D},$ LCB selects the arm using the following rule:
\begin{equation}\label{eqn.lcb.algorithm}
    \widehat{a}_\lcb := \argmax_{a_i \in \cA} ~l_i,
\end{equation}
\citet{rashidinejad2021bridging} analyzed the LCB strategy. Below we provide a modified version of their theorem. 
\begin{theorem}[LCB Performance]\label{thm.LCB}
    Suppose the noise of arm $a_i$ is sub-Gaussian with proxy variance $\sigma_i^2.$
    Let $\delta \in (0,1/2).$ Then, we have
    \begin{enumerate}[leftmargin=*]
        \item (Comparison with any arm) With probability at least $1-\delta,$ for any comparator policy $a_i \in \cA$, it holds that
        \begin{equation}\label{eqn.upper.bound.LCB.first}
            r\l(a_i\r) - r\l(\widehat{a}_\lcb\r) \leq \sqrt{\frac{8\sigma_i^2}{N(a_i)} \log\l(\frac{2d}{\delta}\r)}.
        \end{equation}
        \item (Comparison with the optimal arm) Assume $\sigma_i = 1$ for any $i \in [d]$ and $N \geq 8 C^* \log\l(1/\delta\r).$ Then, with probability at least $1-2\delta,$ one has
        \begin{equation}\label{eqn.upper.bound.LCB.second}
            r\l(a^*\r) - r\l(\widehat{a}_\lcb\r) \leq
            \sqrt{\frac{4 C^*}{N} \log\l(\frac{2d}{\delta}\r)}.
        \end{equation}
    \end{enumerate}    
\end{theorem}
The statement of this theorem is slightly different from that in~\citet{rashidinejad2021bridging}, in the sense that their suboptimality is over $\E_{\mathcal{D}}[r\l(a^*\r) - r\l(\widehat{a}_\lcb\r)]$ instead of a high-probability one.
\citet{rashidinejad2021bridging}  proved the minimax optimality of the algorithm when the single policy concentrability $C^* \geq 2$ and the sample size $N \geq \widetilde{O}(C^*).$ 

\subsection{A data-starved MAB problem and failure of LCB}\label{sec:one.sample}

In order to highlight the limitation of a strategy such as LCB, let us describe a specific data-starved MAB instance, specifically one with $d = 10000$ arms, equally partitioned into a set of \emph{good arms} (i.e., $\cA_{g}$) and a set of \emph{bad arms} (i.e., $\cA_{b}$). 
Each good arm returns a reward following the uniform distribution over $[0.5,1.5],$ 
while each bad arm returns a reward which follows $\normal(0,1/4)$. 

Assume that we are given a dataset that contains 
\textit{only one sample per each arm}. 
Instantiating the LCB confidence interval in \eqref{eqn.confidence.bound.LCB} with $\sigma_i \leq 1/2$ and $\delta = 0.1,$ one obtains  
\begin{equation*}
    \widehat{r}_i - 2.5 \leq r(a_i) \leq \widehat{r}_i + 2.5.
\end{equation*}
Such bound is uninformative, because the lower bound for the true reward mean is less than the reward value of the worst arm. The performance bound for LCB confirms this intuition, because \cref{thm.LCB} requires at least
$N(a_i) \geq \lceil 8 * \log(1/0.05) \rceil = 24$
samples in each arm
to provide any guarantee with probability at
least $0.9$ (here $C^* = d$). 

\subsection{Can stochastic policies help?}
At a first glance, extracting a good decision-making strategy for the problem discussed in \cref{sec:one.sample} seems like a hopeless endeavor, because it is information-theoretically impossible to reliably estimate the expected payoff of any of the arms with just a single sample on each. 

In order to proceed, the key idea is to enlarge the search space to contain \emph{stochastic policies}.
\begin{definition}[Stochastic Policies]\label{def.stochastic.policy}
    A stochastic policy over a MAB is a probability distribution 
    $
    w \in \R^{d}, w_i \geq 0, \sum_{i=1}^d w_i = 1.
    $
\end{definition}
To exemplify how stochastic policies can help,
consider the \emph{behavioral cloning} policy, which mimics the policy that generated the dataset for the offline MAB in \cref{sec:one.sample}. 
Such policy is stochastic, and it plays all arms uniformly at random, thereby achieving a score around $0.5$ with high probability. 
The value of the behavioral cloning policy can be readily estimated using the Hoeffding bound (e.g., Proposition 2.5 in~\citet{wainwright2019high}): with probability at least $1-\delta = 0.9,$ (here $d = 10000$ is the number of arms and $\sigma = 1/2$ is the true standard deviation), the value of behavioral cloning policy is greater or equal than
\begin{align*}
    \frac{1}{2} - \sqrt{\frac{2 \sigma^2 \log\l(2/\delta\r)}{d}} \approx 0.488.
\end{align*}

Such value is higher than the one guaranteed for LCB by \cref{thm.LCB}. 
Intuitively, a stochastic policy that selects multiple arms can be evaluated more accurately because it averages the rewards experienced over different arms. 
This consideration suggests optimizing over stochastic policies.

By optimizing a lower bound on the performance of the stochastic policies, it should be possible to find one with a provably high return.
Such an idea leads to solving an offline \emph{linear bandit} problem, as follows
\begin{align}\label{eqn.linear.bandit}
    \max_{w \in \R^d, w_i \geq 0, \sum_{i=1}^d w_i = 1} 
    \ &\sum_{i=1}^d w_i \widehat r_i - c(w) 
\end{align}
where $c(w)$ is a suitable confidence interval for the policy $w$ and $\widehat r_i$ is the empirical reward for the $i$-th arm defined in \eqref{eqn.emp.reward.mab}.
While this approach is appealing, 
enlarging the search space to include all stochastic policies brings an increase in the metric entropy of the function class, and concretely, a $\sqrt{d}$ factor ~\citep{Abbasi11, rusmevichientong2010linearly, hazan2016volumetric, jun2017scalable, kim2022improved} in the confidence intervals $c(w)$ (in \eqref{eqn.linear.bandit}), which negates all gains that arise from considering stochastic policies.
In the next section, we propose an algorithm that bypasses the need for such $\sqrt{d}$ factor by relying on a more careful analysis and optimization procedure.

\section{Trust Region of Uncertainty for Stochastic policy enhancemenT (TRUST)}\label{sec.TRUST}
In this section, we introduce our algorithm, called Trust Region of Uncertainty for Stochastic policy enhancemenT (TRUST).
At a high level, the algorithm is a policy optimization algorithm based on a trust region centered around a reference policy. The size of the trust region determines the degree of pessimism, and its optimal problem-dependent size can be determined by analyzing the supremum of a problem-dependent empirical process. In the sequel, we describe 1) the decision variables, 2) the trust region optimization program, and 3) some techniques for its practical implementation.

\subsection{Decision variables}
The algorithm searches over the class of stochastic policies given by the weight vector $w = (w_1,w_2,...,w_d)^\top$ of \cref{def.stochastic.policy}. 
Instead of directly optimizing over the weights of the stochastic policy, it is convenient to center $w$  
around a \emph{reference stochastic policy} $\widehat{\mu}$ which is either known to perform well or is easy to estimate. In our theory and experiments, we consider a simple setup and use the behavioral cloning policy weighted by the noise levels $\{\sigma_i\}$ if they are known. Namely, we consider
\begin{equation}\label{eqn.uniform.general}
    \widehat{\mu}_i = \frac{N_i / \sigma_i^2}{\sum_{j=1}^d N_j / \sigma_j^2} \quad \forall i \in [d].
\end{equation}
When the size of the noise $\sigma_i$ 
is constant across all arms, the policy $\widehat{\mu}$ 
is the behavioral cloning policy; when $\sigma_i$ differs across arms, $\widehat{\mu}$ minimizes the variance of the empirical reward
\begin{equation*}
    \widehat{\mu} = \argmin_{w \in \R^d, w_i \geq 0, \sum_i w_i = 1} \operatorname{Var}\l(w^\top \cdot \widehat{r}\r),
\end{equation*}
where $\widehat{r} = (\widehat{r}_1,...,\widehat{r}_d)^\top$ is defined in \eqref{eqn.emp.reward.mab}. Using such definition, we define as \emph{decision variable} the \emph{policy improvement} vector
\begin{equation}\label{eqn.def.policy.improvement}
    \improve := w - \widehat{\mu}.
\end{equation}
This preparatory step is key: 
it allows us to implement \textbf{relative pessimism}, namely pessimism on the improvement---represented by $\improve$---rather than on the absolute value of the policy $w$. Moreover, by restricting the search space to a ball around $\widehat{\mu}$, one can efficiently reduce the metric entropy of the policy class and obtain tighter confidence intervals.
\begin{figure}[H]
    \centering
    \includegraphics[width = 0.45\textwidth]{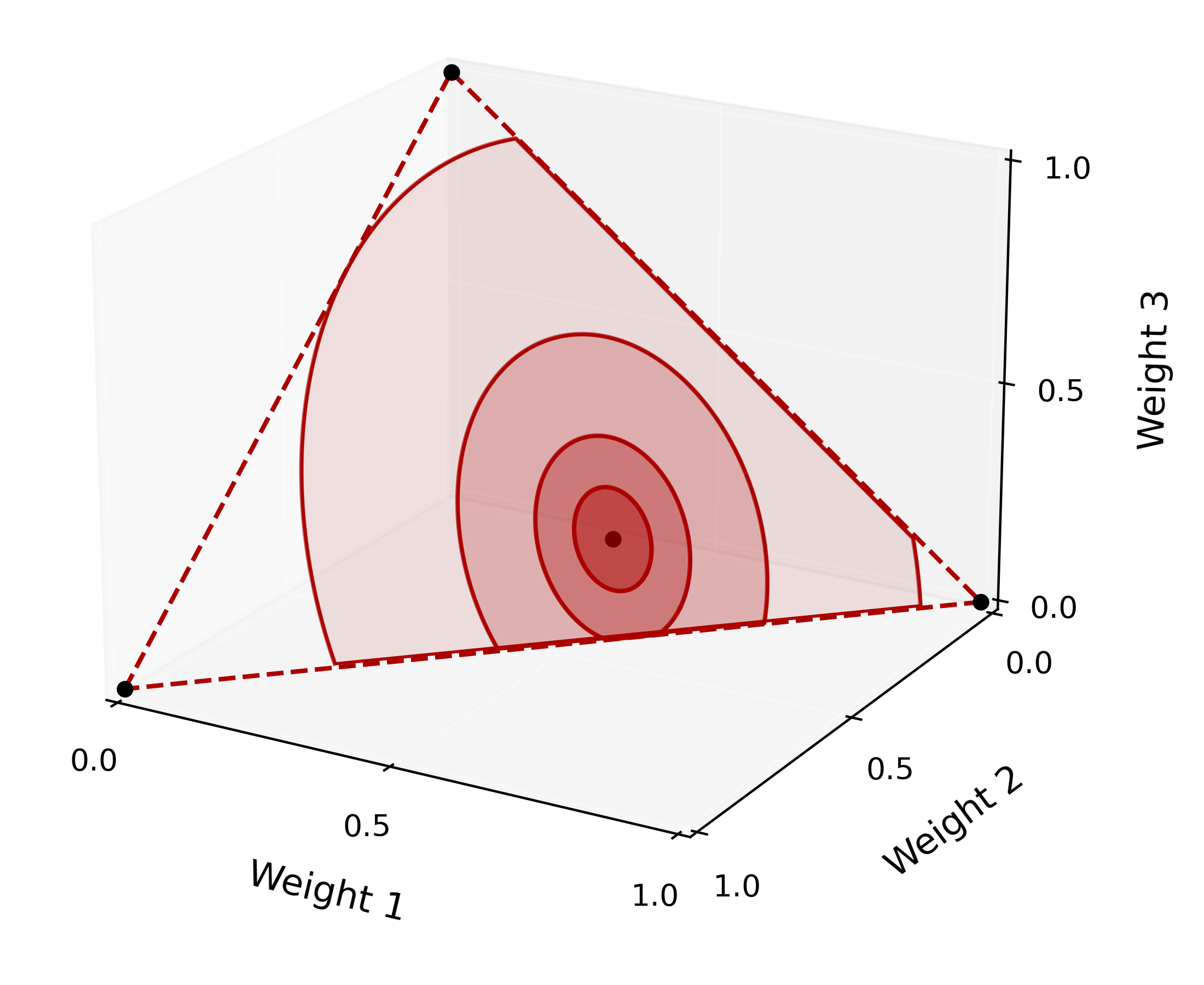}
    \caption{A simple diagram for the trust regions on a $3$-dim simplex. The central point is the reference (stochastic) policy, while red ellipses are trust regions around this reference policy.}
    \label{fig.TR}
\end{figure}

\subsection{Trust region optimization}
\paragraph{Trust region.} TRUST (\cref{alg:protocol}) returns the stochastic policy 
$
\pi_{TRUST} = \widehat{\improve} + \widehat{\mu} \in \R^d,
$
where $\widehat{\mu}$ is the reference policy defined in \eqref{eqn.uniform.general} and $\widehat{\improve}$ is the policy improvement vector. In order to accurately quantify the effect of the improvement vector $\Delta,$ we constrain it to a trust region $\trustRegion{\eps}$ centered around $\widehat \mu$ where $\eps > 0$ is the radius of the trust region. More concretely, for a given radius $\eps > 0,$ the trust region is defined as 
\begin{align}\label{eqn.def.trust.region}
    \trustRegion{\eps} := \Bigg\{\improve : & \improve_i + \widehat{\mu}_i \geq 0, \norm{1}{\improve + \widehat{\mu}} = 1, \notag \\
    & \sum_{i=1}^d \frac{\improve_i^2 \sigma_i^2}{N_i} \leq \eps^2\Bigg\}.
\end{align}
The trust region above serves two purposes: it  
ensures that the policy $\widehat{\improve} + \widehat{\mu}$ still represents a valid stochastic policy, and it regularizes the policy around the reference policy $\widehat{\mu}$.
We then search for the best policy within $\trustRegion{\eps}$ by solving the optimization program
\begin{equation}\label{eqn.optimal.epss.per.stage}
    \empImprove{\eps} := \argmax_{\improve \in \trustRegion{\eps}} \improve^\top \widehat{r}.
\end{equation}
Computationally, the program \eqref{eqn.optimal.epss.per.stage} is a second-order cone program~\citep{alizadeh2003second, boyd2004convex}, which can be solved efficiently with standard off-the shelf libraries~\citep{diamond2016cvxpy}. 

When $\eps = 0$, the trust region only includes the vector $\improve = 0$, and the reference policy $\widehat{\mu}$ is the only feasible solution. When $\eps \rightarrow \infty$, the search space includes all stochastic policies.
In this latter case, the solution identified by the algorithm coincides with the greedy algorithm which chooses the arm with the highest empirical return.
Rather than leaving $\eps$ as a hyper-parameter, in the next section we highlight a selection strategy for $\eps$ based on localized Gaussian complexities.

\paragraph{Critical radius.} The choice of $\eps$ is crucial to the performance of our algorithm because it balances optimization with regularization. Such consideration suggests that there is an optimal choice for the radius $\eps$ 
which balances searching over a larger space with keeping the metric entropy of such space under control.
The optimal problem-dependent choice $\empEps$ 
can be found as a solution of a certain equation involving a problem-dependent \emph{supremum of an empirical process}.
More concretely, let $E$ be the feasible set of $\eps$ (e.g., $E = \mathbb R^+ $). We define the critical radius as
\begin{definition}[Critical Radius]
The critical radius $\empEps$ of the trust region is the solution to the program
\begin{equation}\label{eqn.optimal.eps}
    \empEps = \argmax_{\eps \in E} \l[\empImprove{\eps}^\top \cdot \widehat{r} - \GaussSupremum{\eps}\r].
\end{equation}	
\end{definition}

Such equation involves a quantile of the localized gaussian complexity 
$\GaussSupremum{\eps}$ of 
the stochastic policies identified by the trust region. Mathematically, this is defined as
\begin{definition}[Quantile of the supremum of Gaussian process]\label{def.G}
    We denote the noise vector as $\eta = \widehat{r} - r,$ which by our assumption 
    is coordinate-wise independent and 
    satisfies $\eta_i \sim \normal\l(0,\sigma_i^2/N(a_i)\r).$ We define $\GaussSupremum{\eps}$ as the smallest quantity such that with probability at least $1-\delta$, for any $\eps \in E,$ it holds that
    \begin{equation}\label{eqn.assumption.bonus}
        \sup_{\improve \in \trustRegion{\eps}} \improve^\top \noisevec \leq \GaussSupremum{\eps}.
    \end{equation}
\end{definition}

In essence, $\GaussSupremum{\eps}$ is an upper quantile of the supremum of the Gaussian process $\sup_{\improve \in \trustRegion{\eps}} \improve^\top \noisevec$ which holds uniformly for every $\eps \in E.$ We also remark that this quantity depends on the feasible set $E$ and the trust region $\trustRegion{\eps}$, and hence, is highly problem-dependent.

The critical radius plays a crucial role:
it is the radius of the trust region that \emph{optimally} balances optimization with uncertainty.
Enlarging $\eps$ enlarges the search space for $\Delta$, enabling the discovery of policies with potentially higher return. 
However, this also brings an increase in the metric entropy of the policy class encoded by $\GaussSupremum{\eps}$, which means that each policy can be estimated less accurately.
The critical radius represents the optimal tradeoff between these two forces. 
The final improvement vector that TRUST returns, which we denote as $\widehat \improve_*$, is determined by solving \eqref{eqn.optimal.epss.per.stage} with the critical radius $\widehat \eps_*$. In mathematical terms, we express this as
\begin{equation}\label{def.best.improvement}
    \widehat \improve_* := \argmax_{\improve \in \trustRegion{\widehat \eps_*}} \improve^\top \widehat{r},
\end{equation}
where $\widehat \eps_*$ is defined in \eqref{eqn.optimal.eps}.

\begin{algorithm}[h]
   \caption{Trust Region of Uncertainty for Stochastic policy enhancemenT (TRUST)}
   \label{alg:protocol}
\begin{algorithmic}
    \STATE {\bfseries Input:} Offline dataset $\OfflineDataset,$ failure probability $\delta,$ the candidate set for the trust region widths $E$ (in practice, this is chosen as \eqref{eqn.discrete.set.E}). 
    \STATE 1. For $\eps \in E,$ compute $\empImprove{\eps}$ from \eqref{eqn.optimal.epss.per.stage}. 
    \STATE 2. For $\eps \in E,$ estimate 
    $\GaussSupremum{\eps}$ via Monte-Carlo method (see \cref{alg:monte.carlo} in \cref{sec:monte.carlo}).
    \STATE 3. Solve \eqref{eqn.optimal.eps} to obtain the critical radius $\widehat \eps_*.$
    \STATE 4. Compute the optimal improvement vector in $\trustRegion{\widehat \eps_*}$ via \eqref{def.best.improvement}, denoted as $\selectImprove.$
    \STATE 5. Return the stochastic policy $\pi_{TRUST} = \widehat{\mu} + \selectImprove.$
    \end{algorithmic}
\end{algorithm}

\paragraph{Implementation details}
Since it can be difficult to solve \eqref{eqn.optimal.eps} for a continuous value of $\eps \in E = \mathbb R^+$, we use a discretization argument by considering the following candidate subset:
\begin{equation}\label{eqn.discrete.set.E}
    E = \l\{\eps_0,  \frac{\eps_0}{\alpha}, ... , \frac{\eps_0}{\alpha^{|E|-1}}\r\},
\end{equation}
where $\alpha > 1$ is the decaying rate and $\eps_0$ is the largest possible radius, which is the maximal weighted distance from the reference policy to any vertex. Mathematically, this is defined as
\begin{equation*}
    \eps_0 = \max_{i\in [d]} \sqrt{\sum_{j\neq i} \frac{\widehat \mu_j^2 \sigma_j^2}{N_j} + \frac{\l(1-\widehat \mu_i\r)^2 \sigma_i^2}{N_i}}.
\end{equation*}
Our analysis that leads to \cref{thm.performance.tr} takes into account such discretization argument.

In line 2 of \cref{alg:protocol}, the algorithm works by estimating the quantile of the supremum of the localized Gaussian complexity $\GaussSupremum{\eps}$ that appears in \cref{def.G}, and then choose the $\eps$ that maximizes the objective function in \eqref{eqn.optimal.eps}.
Although $\GaussSupremum{\eps}$ can be upper bounded analytically, in our experiments we aim 
to obtain tighter guarantees and so we estimate it via Monte-Carlo.
This can be achieved by 1) sampling independent noise vectors $\eta$, 2) solving $\sup_{\Delta \in \trustRegion{\eps}} \improve^\top \eta$ and 3) estimating the quantile via order statistics. More details can be found in \cref{sec:monte.carlo}. 

In summary, our practical algorithm can be seen as solving the optimization problem 
$$
(\widehat \eps_*, \selectImprove) = \mathop{\arg \max}_{\eps \in E, \improve \in \trustRegion{\eps}}
\bigg\{ \improve^\top \widehat r - \widehat{\mathcal{G}} (\eps) \bigg\}
$$
where $\widehat r \in \R^d$ is the empirical reward vector with $\widehat r_i$ defined in \eqref{eqn.emp.reward.mab}. Here,
$\widehat{\mathcal{G}}(\eps)$ is computed according to the Monte-Carlo method defined in \cref{alg:monte.carlo} in \cref{sec:monte.carlo} and $E$ is the candidate subset for radius defined in \eqref{eqn.discrete.set.E}. This indicates a balance between the empirical reward of a stochastic policy and the local entropy metric it induces, representing 

\section{Theoretical guarantees}
In this section, we provide some theoretical guarantees for the policy $\pi_{TRUST}$ returned by TRUST. 

\subsection{Problem-dependent analysis}
We present 1) an improvement over the reference policy $\widehat \mu$, 2) a sub-optimality gap with respect to any comparator policy $\pi$ and 3) an actionable lower bound on the performance of the output policy.

Given a stochastic policy $\pi$, we let $V^\pi = \mathbb E_{a \sim \pi}[r(a)]$ denote its value function.
Furthermore, we denote a comparator policy $\pi$
by a triple $(\eps,\improve,\pi)$ such that $\eps > 0, \improve \in \trustRegion{\eps}, \pi = \widehat{\mu} + \improve.$

\begin{theorem}[Main theorem]\label{thm.performance.tr}
   TRUST has the following properties.
\begin{enumerate}[leftmargin=*]
    \item With probability at least $1-\delta,$ the improvement over the behavioral policy is at least
    \begin{equation}\label{eqn.lower.bound.improvement}
        V^{\pi_{TRUST}} - V^{\widehat \mu} 
        \geq \sup_{\eps \leq \eps_0, \improve \in \trustRegion{\eps}} \l[\improve^\top r - 2 \GaussSupremum{\ceil{\eps}}\r],
    \end{equation}
    where $\lceil \eps \rceil = \inf\{\eps^\prime \in E, \eps^\prime \geq \eps\}.$
	\item  With probability at least $1-\delta,$ for any stochastic comparator policy $(\eps,\improve,\pi)$, the sub-optimality of the output policy can be upper bounded as
    \begin{equation}\label{eqn.suboptimality.gap}
       V^{\pi} - V^{\pi_{TRUST}} 
        \leq 
        2 \GaussSupremum{\ceil{\eps}}.
    \end{equation}
    \item With probability at least $1-2\delta,$ the data-dependent lower bound on $V^{\pi_{TRUST}}$ satisfies
    \begin{equation}\label{eqn.computable.lower.bound}
        V^{\pi_{TRUST}} \geq \pi_{TRUST}^\top \widehat{r} - \GaussSupremum{\lceil \widehat \eps_* \rceil} - \sqrt{\frac{2\log(1/\delta)}{\sum_{j=1}^d N_j / \sigma_j^2}},
    \end{equation}
    where $\pi_{TRUST} = \widehat{\mu} + \selectImprove$ is the policy output by \cref{alg:protocol}.
    \end{enumerate}
\end{theorem}

Our guarantees are problem-dependent as a function of the Gaussian process $\GaussSupremum{\cdot}$; in \cref{sec.simulated.experiment} we show how these can be instantiated on an actual problem, highlighting the tightness of the analysis. 

\cref{eqn.lower.bound.improvement} highlights the 
improvement with respect to the behavioral policy. It is expressed as a trade-off between maximizing the improvement $\Delta^\top r$ and minimizing its uncertainty $\GaussSupremum{\ceil{\eps}}$.
The presence of the $\sup_\eps$ indicates that TRUST achieves an \emph{optimal} balance between these two factors.
The state of the art guarantees that we are aware of highlight a trade-off between value and variance \citep{jin2021pessimism,min2021variance}. 
The novelty of our result lies in the fact that TRUST optimally balances the uncertainty implicitly as a function of the `coverage' as well as the metric entropy of the search space. That is, TRUST selects the most appropriate search space by trading off its metric entropy with the quality of the policies that it contains.

The right-hand side in \cref{eqn.computable.lower.bound}
gives actionable statistical guarantees on the quality of the final policy and it can be fully computed from the available dataset;
we give an example of the tightness of the analysis in \cref{sec.simulated.experiment}.

\paragraph{Localized Gaussian complexity $\GaussSupremum{\eps}$.}
In \cref{thm.performance.tr}, we upper bound the suboptimality $V^\pi - V^{\pi_{TRUST}}$ via a notion of localized metric entropy $\GaussSupremum{\cdot}.$ It is the quantile of the supremum of a Gaussian process, which can be efficiently estimated via Monte Carlo method (e.g., see \cref{sec:monte.carlo}) or concentrated around its expectation. The expected value of $\GaussSupremum{\eps}$ is also \emph{localized Gaussian width}, a concept well-established in statistical learning theory~\citep{bellec2019localized,wei2020gauss,wainwright2019high}. More concretely, this is the localized Gaussian width for an affine simplex:
\begin{equation*}
    \E \l[\sup_{\improve \in \trustRegion{\eps}} \improve^\top \noisevec \r] = 
    \E \l[\sup_{\mathbb{S}^{d-1} \cap \l\{\improve: \norm{\Sigma}{\improve} \leq \eps \r\}} \improve^\top \noisevec\r],
\end{equation*}
where $\mathbb{S}^{d-1}$ denotes the simplex in $\R^d$ and $\Sigma := \diag{\frac{\sigma_1^2}{N_1}, \frac{\sigma_2^2}{N_2}, ..., \frac{\sigma_d^2}{N_d}}$ is the weighted matrix. Moreover, this localized Gaussian width can be upper bound via
\begin{equation}\label{eqn.upper.bound.localized.Gaussian}
    \E \l[\sup_{\improve \in \trustRegion{\eps}} \improve^\top \noisevec\r] \lesssim \min\l\{\sqrt{\log\l(d\eps^2\r)}, \eps \sqrt{d}\r\}.
\end{equation}

\begin{figure}[htbp]
    \centering
    \includegraphics[width = 0.45\textwidth]{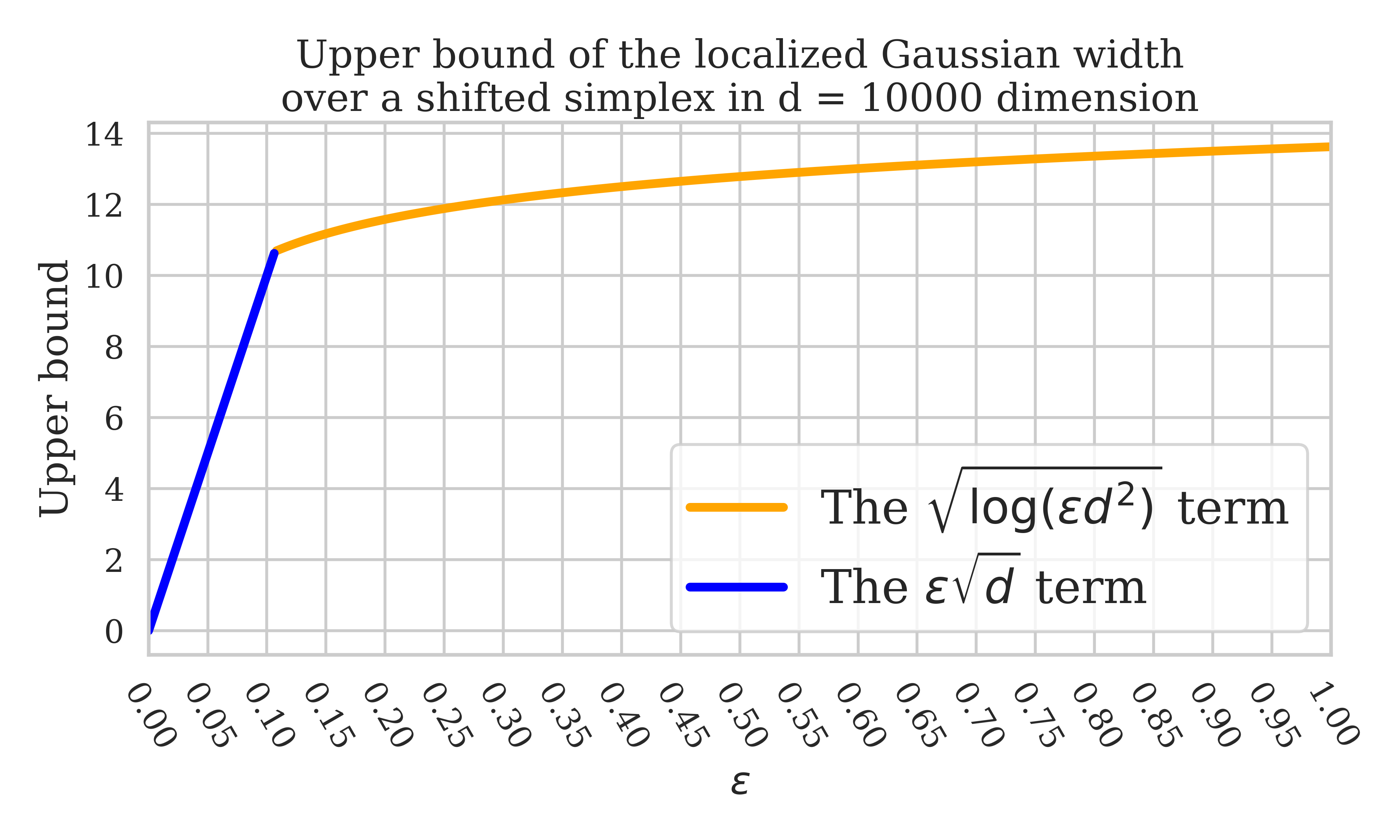}
    \caption{The upper bound for the localized Gaussian width over a shifted simplex on $d=10000$ dimension. The shifted simplex is $\l\{ \improve \in \R^d: \sum_{i=1}^d \improve_i = 0\r\}.$ The two-staged upper bound we plot is based on Theorem 1 in \citep{bellec2019localized}}
    \label{fig.local.Gaussian.width}
\end{figure}

To make it clearer, we plot this upper bound for localized Gaussian width in \cref{fig.local.Gaussian.width}. 
In \eqref{eqn.upper.bound.localized.Gaussian}, the rate matches the minimax lower bound up to universal constant~\citep{gordon2007gaussian,lecue2013learning,bellec2019localized}. To see the implication of the upper bound \eqref{eqn.upper.bound.localized.Gaussian}, let's consider a simple example where the logging policy is uniform over all arms. We denote the optimal arm as $a^*$ and define
 \begin{equation*}
    C^* := \frac{1}{\mu(a^*)}
\end{equation*}
as the concentrability coefficient. By applying \eqref{eqn.upper.bound.localized.Gaussian} and some concentration techniques \citep[see][]{wainwright2019high}, we can perform a fine-grained analysis for the suboptimality induced by $\pi_{TRUST}.$ Specifically, with probability at least $1-\delta,$ one has 
\begin{equation}
    V^{\pi_*} - V^{\pi_{TRUST}} 
    \lesssim \sqrt{\frac{C^*}{N} \log\l(\frac{2 d |E|}{\delta}\r)}.
\end{equation}
Note that, the high-probability upper bound here is minimax optimal up to constant and logarithmic factor \citep{rashidinejad2021bridging} when $C^* \geq 2.$ Moreover, this example of uniform logging policy is an instance where LCB achieves 
minimax sub-optimality (up to constant and log factors) \citep[see the proof of Theorem 2 in][]{rashidinejad2021bridging}. In this case, TRUST will achieve the same level of guarantees for the suboptimality of the output policy. We also empirically show the effectiveness of TRUST in \cref{sec.simulated.experiment}. The full theorem for a fine-grained analysis for the suboptimality and its proof are deferred to \cref{appendix.fine.grain.analysis}.

\subsection{Proof of \cref{thm.performance.tr}}
    To prove \cref{thm.upper.bound.G}, we first define the following event
    \begin{equation}
        \event := \l\{\sup_{\improve \in \trustRegion{\eps}} \improve^\top \noisevec \leq \GaussSupremum{\eps} \quad \forall \eps \in E\r\}.
    \end{equation}
    When $\event$ happens, the quantity $\GaussSupremum{\eps}$ can upper bound the supremum of the Gaussian process we care about, and hence, we can effectively upper bound the uncertainty for any stochastic policy using $\GaussSupremum{\cdot}.$ It follows from the \cref{def.G} that the event $\event$ happens with probability a leas $1-\delta.$ 
    
    We can now prove all the claims in the theorem, starting from the first and the second. A comparator policy $\pi$ identifies a weight vector $w$, an improvement vector $\Delta$ and a radius $\eps$ such that $w = \widehat{\mu} + \improve$ and $\improve \in \trustRegion{\eps}.$ In fact, we can always take $\eps$ to be the minimal value such that $\improve \in \trustRegion{\eps}.$ The first claim in \cref{eqn.lower.bound.improvement} can be proved by establishing that with probability at least $1-\delta$
     \begin{equation}
     \label{eqn:proof.first.statement}
        w^\top r - \pi^\top_{TRUST} r = \improve^\top r - \selectImprove^\top r \leq 2 \GaussSupremum{\lceil \eps \rceil},
    \end{equation}
    where $\pi_{TRUST}$ is the policy weight returned by \cref{alg:protocol}. In order to show \cref{eqn:proof.first.statement}, we can decompose $\selectImprove^\top r $ using the fact that $\empEps \in E$ and $ \selectImprove \in \trustRegion{\empEps}$ to obtain 
    \begin{align}
        \selectImprove^\top r 
        &= 
        \selectImprove^\top \widehat{r} - \selectImprove^\top \eta 
        \geq \selectImprove^\top \widehat{r} - \GaussSupremum{\empEps}
        = \selectImprove^\top \widehat{r} - \GaussSupremum{\ceil{\empEps}} \label{eqn.main.proof1}.
    \end{align}
    To further lower bound the RHS above, we have the following lemma, which shows that \cref{alg:protocol} can be written in an equivalent way.
    \begin{lemma}\label{lem.equivalent.form}
        The output of \cref{alg:protocol} satisfies
        \begin{equation}
            \l(\widehat \eps_*, \selectImprove\r)
        = \argmax_{\eps \leq \eps_0, \improve \in \trustRegion{\eps}} \big[\improve^\top \empR - \GaussSupremum{\lceil\eps\rceil} \big].
        \end{equation}
    \end{lemma}
    This shows that \cref{alg:protocol} optimizes over an objective function which consists of a signal term (i.e., $\improve^\top \widehat{r}$) minus a noise term (i.e., $\GaussSupremum{\lceil \eps \rceil}$). Applying this lemma to \eqref{eqn.main.proof1}, we know
    \begin{align}\label{eqn.main.proof4}
        \selectImprove^\top r &\geq \improve^\top \widehat{r} - \GaussSupremum{\ceil{\eps}} 
        = \improve^\top r + \improve^\top \eta - \GaussSupremum{\ceil{\eps}}. 
    \end{align}
    After recalling that under $\event$
    \begin{align}\label{eqn.main.proof3}
        \improve^\top \eta 
        \leq \sup_{\improve \in \trustRegion{\eps}} \improve^\top \eta 
        \leq \sup_{\improve \in \trustRegion{\lceil\eps\rceil}} \improve^\top \eta 
        \leq \GaussSupremum{\lceil\eps\rceil},
    \end{align}
    plugging the \eqref{eqn.main.proof3} back into \eqref{eqn.main.proof4} concludes the bound in \cref{eqn:proof.first.statement}, which also proves our first claim.
    Rearranging the terms in \cref{eqn:proof.first.statement} and taking supremum over all comparator policies, we obtain
    \begin{equation}\label{eqn.main.proof2}
        \selectImprove^\top r \geq \sup_{\eps \leq \eps_0, \improve \in \trustRegion{\eps}} \l[\improve^\top r - 2 \GaussSupremum{\ceil{\eps}}\r],
    \end{equation}
    which proves the first claim since $V^{\pi_{TRUST}} - V^{\widehat \mu} = \selectImprove^\top r.$
    
  In order to prove the last claim, it suffices to lower bound the policy value of the reference policy $\widehat{\mu}.$ From \eqref{eqn.uniform.general}, we have $\widehat{\mu}\l(\widehat r - r\r) \sim \normal(0, 1/[\sum_{i=1}^d N_i / \sigma_i^2]),$ which implies with probability at least $1-\delta,$ 
    \begin{equation}\label{eqn.lower.bound.reference}
        \widehat{\mu}\l(\widehat r - r\r) \leq \sqrt{\frac{2\log(1/\delta)}{\sum_{i=1}^d N_i / \sigma_i^2}}
    \end{equation}
    from the standard Hoeffding inequality (e.g., Prop 2.5 in~\citet{wainwright2019high}). Combining \eqref{eqn.main.proof1} and \eqref{eqn.lower.bound.reference}, we obtain  \begin{align*}
        \pi_{TRUST}^\top r 
        &= \widehat \mu^\top r + \selectImprove^\top r \\
        &\geq \widehat \mu^\top \widehat r + \widehat \mu^\top (r-\widehat r) + \selectImprove^\top \widehat r - \GaussSupremum{\widehat \eps_*} \tag{From \eqref{eqn.main.proof1}} \\
        &\geq \pi_{TRUST}^\top \widehat r - \GaussSupremum{\widehat \eps_*} - \sqrt{\frac{2\log(1/\delta)}{\sum_{i=1}^d N_i / \sigma_i^2}} \tag{From \eqref{eqn.lower.bound.reference}}
    \end{align*}
    with probability at least $1-2\delta.$ Therefore, we conclude.

\paragraph{Augmentation with LCB}
Compared to classical LCB, \cref{alg:protocol} considers a much larger searching space, which encompasses not only the vertices of the simplex but the inner points as well. This enlargement of searching space shows great advantage, but this also comes with the price of larger uncertainty, especially when the width $\eps$ is large. In LCB, one considers the uncertainty by upper bound the noise at each vertex uniformly, while in our case, the uniform upper bound for a sub-region of the shifted simplex must be considered. 
When $\eps$ is large, the trust region method will induce larger uncertainty and tend to select a more stochastic policy than LCB and hence, can achieve worse performance. 
To determine the most effective final policy, one can always combine TRUST (\cref{alg:protocol}) with LCB and select the better one between them based on the lower bound induced by two algorithms. By comparing the lower bounds of LCB and TRUST, the value of the finally output policy is guaranteed to outperform the lower bound for either LCB or TRUST with high probability. We defer the detailed algorithm and its theoretical guarantees to \cref{sec.combination.LCB}. 

\section{Experiments}\label{sec.simulated.experiment}
We present simulated experiments where we show the failure of LCB and the strong performance of TRUST. Moreover, we also present an application of TRUST to offline reinforcement learning.
\subsection{Simulated experiments}

\paragraph{A data-starved MAB}
We consider a data-starved MAB problem
with $d=10000$ arms denoted by $a_i, i \in [d]$.
The reward distributions are
\begin{equation}\label{eqn.reward.distribution}
    r(a_i) \sim \l\{
    \begin{aligned}
        &\mathsf{Uniform}(0.5,1.5)   & i \leq 5000, \\
        &\normal\l(0,1/4\r) & i > 5000.
    \end{aligned}
    \r.
\end{equation}
Namely, the set of good arms have reward random variables from a uniform distribution over $[0.5,1.5]$ with unit mean, while the bad arms return a Gaussian reward with zero mean. 
We consider a dataset that contains a single sample 
for each of these arms.

We test \cref{alg:protocol} on this MAB instance with fixed variance level $\sigma_i = 1/2$. We set the reference policy $\widehat \mu$ to be the behavioral cloning policy, which coincides with the uniform policy. We also test LCB and the greedy method which simply chooses the policy with the highest empirical reward.

In this example, the greedy algorithm fails because it erroneously selects an arm with a reward $> 1.5$, but such reward can only originate from a normal distribution with mean zero. Despite LCB incorporates the principle of pessimism under uncertainty,
it selects an arm with average return equal to zero; its performance lower bound given by the confidence intervals is $-1.5,$ which is almost vacuous and very uninformative.
The behavioral cloning policy performs better, because it selects an arm uniformly at random, achieving the score $0.5$. 

\begin{table}[H]
    \centering
    \begin{tabular}{|m{1.5cm}|m{1.5cm}|m{1cm}|m{2cm}|}
    \hline
         \centering Behavior Policy & \centering Greedy Method & \centering\arraybackslash LCB & \centering\arraybackslash LCB Lower Bound\\
    \hline
        \centering 0.5 & \centering 0 & \centering\arraybackslash 0 & \centering\arraybackslash -1.5 \\
    \hline
    \end{tabular}

    \vspace{5mm} 
    \begin{tabular}{|m{1cm}|m{2cm}|m{1.5cm}|m{2cm}|}
    \hline
        \centering max reward & \centering Policy Improvement by TRUST & \centering\arraybackslash TRUST 
            & \centering\arraybackslash TRUST Lower Bound\\
    \hline
        \centering 1.0 & \centering 0.42 & \centering\arraybackslash 0.92 & \centering\arraybackslash 0.6  \\
    \hline
    \end{tabular}
    
    \caption{Results of simulated experiments in a 10000-arm bandit. The reward distribution is described in \eqref{eqn.reward.distribution}. The offline dataset includes one sample for each arm. The greedy method chooses the arm with the highest empirical reward. LCB selects an arm based on \eqref{eqn.lcb.algorithm}. The lower bound for LCB and TRUST follow \eqref{eqn.confidence.bound.LCB} and \eqref{eqn.computable.lower.bound}, respectively.}
    \label{tab:simulated.experiments}
\end{table}

\cref{alg:protocol} achieves the best performance: the value of the policy that it identifies is $0.92,$ which \emph{almost matches the optimal policy}. 
The lower bound on its performance computed by instantiating the RHS in \eqref{eqn.computable.lower.bound} is around $0.6$, a guarantee much tighter than that for LCB.

In order to gain intuition on the learning mechanics of TRUST, in \cref{fig.lower.bounds} we progressively enlarge the radius of the trust region from zero to the largest possible radius (on the $x$ axis) and plot the value of the policy that maximizes the linear objective $\Delta^\top \widehat r, \; \Delta \in \trustRegion{\eps}$ for each value of the radius $\eps$. Note that we rescale the range of $\eps$ to make the largest possible $\eps$ be one.
In the same figure we also plot the lower bound computed with the help of equation 
\eqref{eqn.computable.lower.bound}. 

\begin{figure}[htbp]
    \centering
    \includegraphics[width = 0.45\textwidth]{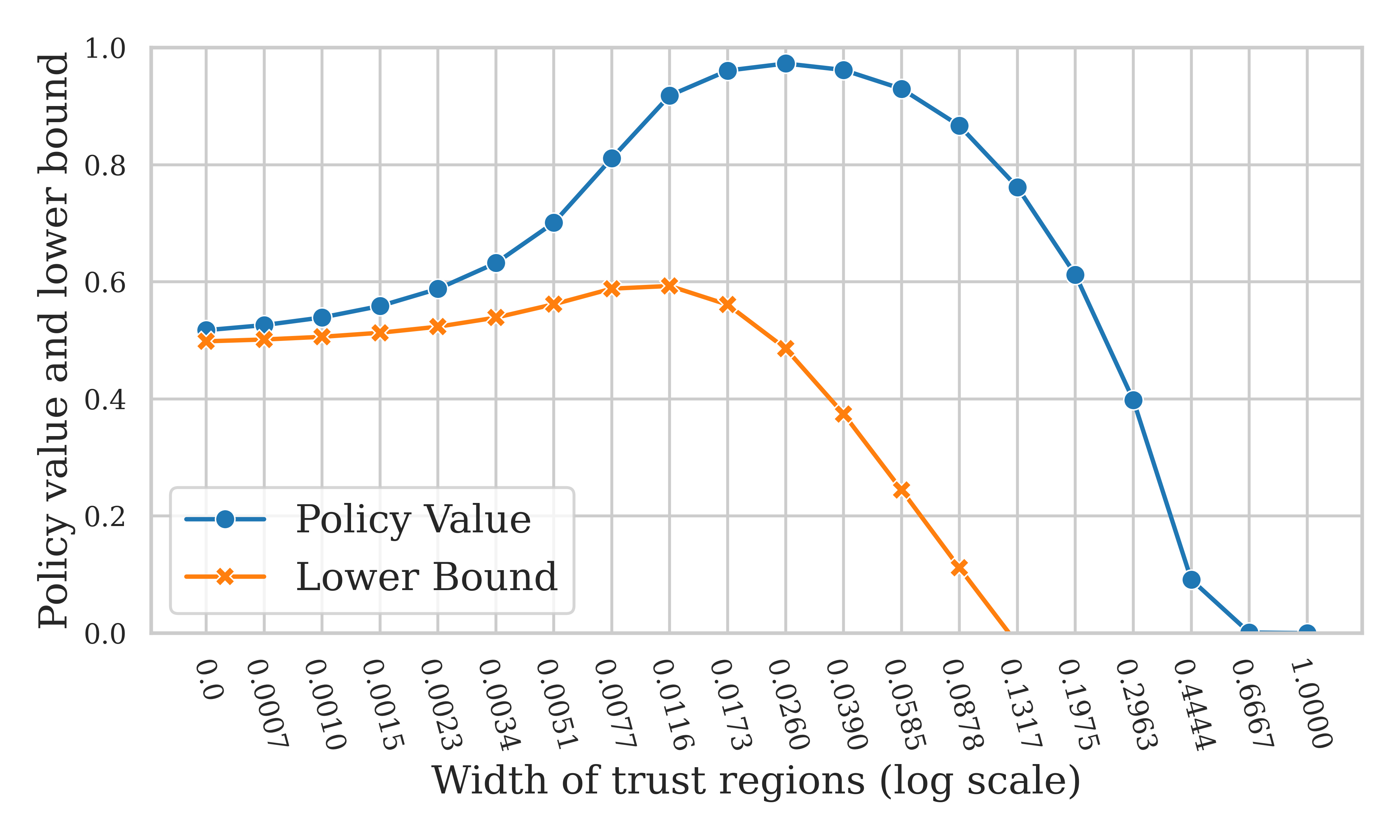}
    \caption{Policy values and their lower bounds for a data-starved MAB instance with 10000 arms whose reward distribution is described in \eqref{eqn.reward.distribution}.}
    \label{fig.lower.bounds}
\end{figure}

Initially, the value of the policy increases because the optimization in \eqref{eqn.optimal.epss.per.stage} is performed over a larger set of stochastic policies. 
However, when $\eps $ approaches one, all stochastic policies are included in the optimization program. In this case, TRUST greedily selects the arm with the highest empirical reward which is from a normal distribution with mean zero.
The optimal balance between the size of the policy search space and its metric entropy is given by the critical radius $\eps = 0.0116 \eps_0$, which is the point where the lower bound is the highest.

\paragraph{A more general data-starved MAB}
Besides the data-starved MAB we constructed, we also show that in general MABs, the performance of TRUST is on par with LCB, but TRUST will have a much tighter statistical guarantee, i.e., a larger lower bound for the value of the returned policy. We did experiments on a $d=1000$-arm MAB where the reward distribution is
\begin{equation}\label{eqn.reward.distribution.general}
    r(a_i) \sim \normal\l(\frac{i}{1000}, \frac{1}{4}\r), \quad \forall i \in [d].
\end{equation}
We ran TRUST \cref{alg:protocol} and LCB over 8 different random seeds.
When we have a single sample for each arm, TRUST will get a similar score as LCB. However, TRUST give a much tighter statistical guarantee than LCB, in the sense that the lower bound output by TRUST is much higher than that output by LCB so that TRUST can output a policy that is guaranteed to achieved a higher value. Moreover, we found the policies output from TRUST are much more stable than those from LCB. In all runs, while the lowest value of the arm chosen by LCB is around 0.24, all policies returned by TRUST have values above 0.65 with a much smaller variance, as shown in \cref{table.general.bandit}.

\begin{table}[H]
\centering
\begin{tabular}{|C{2.7cm}|C{1cm}|C{1cm}|}
\hline
&  LCB & TRUST \\
\hline
mean reward & 0.718 & 0.725 \\
\hline
mean lower bound & 0.156 & 0.544 \\
\hline
variance & 0.265 & 0.038 \\
                        \hline
minimal reward & 0.239 & 0.658 \\
\hline
\end{tabular}
\caption{Comparison between LCB and TRUST (\cref{alg:protocol}) on a data-starved MAB with 1000 arms whose reward distribution follows \eqref{eqn.reward.distribution.general}. Both methods are repeated on 8 random seeds.}
\label{table.general.bandit}
\end{table}

\subsection{Offline reinforcement learning}\label{sec.experiment.rl}
In this section, we apply \cref{alg:protocol} to the offline reinforcement learning (RL) setting under the assumption 
that the logging policies which generated the dataset are accessible. 
To be clear, our goal is not to exceed the performance of the state of the art deep RL algorithms---our algorithm is designed for bandit problems---but rather to illustrate the usefulness of our algorithm and theory.

Since our algorithm is designed for bandit problems, in order to apply it to the sequential setting, we map MDPs to MABs.
Each policy in the MDP maps to an action in the MAB, and each trajectory return  in the MDP maps to an experienced return in the MAB setting.
Notice that this reduction disregards the sequential aspect of the problem and thus our algorithm cannot perform `trajectory stitching'~\citep{levine2020offline,kumar2020conservative,kostrikov2021offline}.
Furthermore, it can only be applied under the assumption that the logging policies are known.

Specifically we consider a setting where there are multiple known
logging policies, each generating few trajectories. 
We test \cref{alg:protocol} on some selected environments from the D4RL dataset~\citep{fu2020d4rl} and compare its performance to the (CQL) algorithm~\citep{kumar2020conservative}, a popular and strong baseline for offline RL algorithms. 

Since the D4RL dataset does not directly include the logging policies, 
we generate new datasets by running Soft Actor Critic (SAC)~\citep{haarnoja2018soft} for 1000 episodes.
We store 100 intermediate policies generated by SAC, 
and roll out 1 trajectory from each policy. 

We use some default hyper-parameters for CQL.\footnote{We use the codebase and default hyper-parameters in \url{https://github.com/young-geng/CQL}.}
We report the unnormalized scores in \cref{table.rl}, 
each averaged over 4 random seeds. 
\cref{alg:protocol} achieves a score on par with or higher than that of CQL, especially when the offline dataset is of poor quality and when there are very few---or just one---trajectory generated from each logging policy. 
Notice that while CQL is not guaranteed to outperform the behavioral policy, TRUST is backed by \cref{thm.performance.tr}.
Moreover, while the performance of CQL is highly reliant on the choice of hyper-parameters,  
TRUST is essentially hyper-parameters free. 

\begin{table}[H]
\centering
\begin{tabular}{|c|C{2cm}|C{1cm}|C{1cm}|}
\hline
\multicolumn{2}{|c|}{} &  CQL & TRUST \\
\hline
\multirow{2}{*}{Hopper} & 1-traj-low & 499 & 999\\
                        & 1-traj-high & 2606 & 3437 \\
                        \hline
\multirow{2}{*}{Ant} & 1-traj-low &  748 & 763 \\
                        & 1-traj-high & 4115 & 4488 \\
                        \hline
\multirow{2}{*}{Walker2d} & 1-traj-low & 311 & 346 \\
                        & 1-traj-high & 4093 & 4097 \\
                        \hline
\multirow{2}{*}{HalfCheetah} & 1-traj-low &  5775 & 5473 \\
                        & 1-traj-high & 9067 & 10380 \\
                        \hline
\end{tabular}
\caption{Unnormalized score of CQL and TRUST in 4 environments from D4RL. In 1-traj-low case, we take the first 100 policies in the running of SAC. In 1-traj-high case, we take the $(10x+1)$-th policy for $x \in [100]$. We sample one trajectory from each policy we take in all experiments.}
\label{table.rl}
\end{table}

Additionally, while CQL took around 16-24 hours on one NVIDIA GeForce RTX 2080 Ti, TRUST only took 0.5-1 hours on 10 CPUs. The experimental details are included in \cref{sec.experiment.details}.

\section{Conclusion}
In this paper we make a substantial contribution towards sample efficient decision making, by designing a data-efficient policy optimization algorithm that leverages offline data for the MAB setting.
The key intuition of this work is to search over stochastic policies, which can be estimated more easily than deterministic ones.

The design of our algorithm is enabled by a number of key insights, such as the use of the localized gaussian complexity which leads to the definition of the critical radius for the trust region. 

We believe that these concepts can be used more broadly to help design truly sample efficient algorithms, which can in turn enable the application of decision making to new settings where a high sample efficiency is critical.

\section{Impact Statement}
This paper presents a work whose goal is to advance the field of decision making under uncertainty. Since our work is primarily theoretical, 
we do not anticipate negative societal consequences.

\bibliographystyle{icml2023}
\bibliography{arxiv}

\appendix
\onecolumn
\section{One-sample case with strong signals}\label{appendix.proof.strong.signal}
In this section, we give a simple example of one-sample-per-arm case. This can be view as a special case of data-starved MAB and \cref{thm.performance.tr} can be applied to get a non-trivial guarantees. Specifically, consider an MAB with $2d$ arms. Assume the true mean reward vector is $r=(1,1,...,1,0,0,...,0)^\top$ and the noise vector is $\eta \sim \normal(0,\sigma^2 I_{2d})$ That is, the first $d$ arms have rewards independently sampled from $\normal(1,1)$ and the rewards for other $d$ arms are independently sampled from $\normal(0,0).$ The stochastic reference policy is set to the uniform one, i.e., $\widehat{\mu} = (\frac{1}{d}, \frac{1}{d},... \frac{1}{d})^\top.$ 

We apply \cref{alg:protocol} to this MAB instance. In the next theorem, we will show that for a specific $\eps,$ the optimal improvement in $\trustRegion{\eps}$ (denoted as $\empImprove{\eps}$ in \eqref{eqn.optimal.epss.per.stage}) can achieve an improved reward value of constant level. 
\begin{proposition}\label{thm.strong.signal}
    Assume $r = (1,1,...,1,0,0,...,0)^\top$ and noise $\eta \sim \normal(0,I_{2d}).$ For any $0 \leq \eps \leq \frac{1}{\sqrt{d}},$ with probability at least $1-\delta,$ the improvement of policy value can be lower bounded by
    \begin{equation*}
        \empImprove{\eps}^\top r \geq \eps \sqrt{d} \l[\frac{1}{2} - \sigma \l(1 + \sqrt{\frac{8 \log\l(2/\delta\r)}{d}}\r)\r],
    \end{equation*}
    where the improvement vector in $\trustRegion{\eps}$ is defined in \eqref{eqn.optimal.epss.per.stage}.
    Therefore, for $\eps = \frac{1}{\sqrt{d}}$ and $d \geq 8 \log\l(2/\delta\r),$ with probability at least $1-\delta,$ we can get a constant policy improvement 
    \begin{equation*}
        \empImprove{\eps}^\top r \geq \frac{1}{2} - 2\sigma.
    \end{equation*}
\end{proposition}

\begin{proof}
    We define the optimal improvement vector as
\begin{equation*}
    \optImprove{\eps} := \mathop{\arg\max}_{\improve \in \trustRegion{\eps}} \improve^\top r.
\end{equation*}
Then, from the definition of $\empImprove{\eps}$, we have
\begin{align}
    \empImprove{\eps}^\top r 
    = \empImprove{\eps}^\top \empR - \empImprove{\eps}^\top \noisevec 
    &\geq \l(\optImprove{\eps}\r)^\top \empR - \empImprove{\eps}^\top \noisevec
    = \l(\optImprove{\eps}\r)^\top r + \l(\optImprove{\eps}\r)^\top \noisevec - \empImprove{\eps}^\top \noisevec
    \geq \underbrace{\l(\optImprove{\eps}\r)^\top r}_{\text{signal}} - \underbrace{\l[\sup_{\improve \in \trustRegion{\eps}} \improve^\top \eta - \l(\optImprove{\eps}\r)^\top \noisevec\r]}_{\text{noise}}. \label{eqn.signal.noise.decom}
\end{align}
In order to lower bound the policy value improvement, it suffices to lower bound the signal part and upper bound the noise. We denote $\hyperplane = \{x \in \R^d: \sum_{i=1}^d x_i = 0\}$  as a hyperplane in $\R^d.$ To deal with the signal part, it suffices to notice that
\begin{equation*}
    \trustRegion{\eps} \subset \hyperplane \cap \ball_2^d(\eps).
\end{equation*}
We denote $r_\parallel$ as the orthogonal projection of $r$ on the $\hyperplane$ and $r_\perp = r - r_\parallel.$ In the strong signal case, we have
\begin{equation*}
    r_\parallel = \l(\frac{1}{2},\frac{1}{2},...,\frac{1}{2}, -\frac{1}{2}, -\frac{1}{2},..., -\frac{1}{2}\r)^\top, \quad 
    r_\perp = \l(\frac{1}{2},\frac{1}{2},...,\frac{1}{2}, \frac{1}{2}, \frac{1}{2},..., \frac{1}{2}\r)^\top.
\end{equation*}
Then, the signal part satisfies
\begin{align}\label{eqn.strong.signal}
    \sup_{\improve \in \trustRegion{\eps}} \improve^\top r
    = \sup_{\improve \in \trustRegion{\eps}} \improve^\top r_\parallel 
    \leq \sup_{\improve \in \hyperplane \cap \ball_2^d(\eps)} \improve^\top r_\parallel 
    = \l(\eps \cdot \frac{r_\parallel}{\norm{2}{r_\parallel}}\r)^\top r_\parallel
    = \eps \norm{2}{r_\parallel} 
    = \frac{\eps \sqrt{d}}{2}.
\end{align}
On the other hand, we notice that when $\eps \leq \frac{1}{\sqrt{d}},$
\begin{equation*}
    \eps \cdot \frac{r_\parallel}{\norm{2}{r_\parallel}} = \l(\frac{\eps}{\sqrt{d}}, \frac{\eps}{\sqrt{d}}, ..., \frac{\eps}{\sqrt{d}}, - \frac{\eps}{\sqrt{d}}, - \frac{\eps}{\sqrt{d}}, ..., -\frac{\eps}{\sqrt{d}}\r)^\top \in \trustRegion{\eps}.
\end{equation*}
So actually the inequality in the \eqref{eqn.strong.signal} should be an equation, which implies
\begin{equation}\label{eqn.bound.signal}
    \sup_{\improve \in \trustRegion{\eps}} \improve^\top = \frac{\eps \sqrt{d}}{2}.
\end{equation}

\ 

For the noise part, we decompose the noise as $\eta = \eta_{\perp} + \eta_{\parallel},$ where $\eta_\parallel$ is the orthogonal projection of $\eta$ on $\hyperplane.$ Then, from $\trustRegion{\eps} \subset \hyperplane \cap \ball_2^d(\eps),$ one has
\begin{align*}
    \sup_{\improve \in \trustRegion{\eps}} \improve^\top \eta
    & = \sup_{\improve \in \trustRegion{\eps}} \improve^\top \l(\eta_\parallel + \eta_\perp\r) 
    = \sup_{\improve \in \trustRegion{\eps}} \improve^\top \eta_\parallel
    \leq \sup_{\improve \in \hyperplane \cap \ball_2^d(\eps)} \improve^\top \eta_\parallel \\
    &= \l(\eps \cdot \frac{\eta_\parallel}{\norm{2}{\eta_\parallel}}\r)^\top \eta_\parallel
    = \eps \norm{2}{\eta_\parallel} 
    \leq \eps \norm{2}{\eta}.
\end{align*}
This implies $\empImprove{\eps}^\top r \geq \l(\optImprove{\eps}\r)^\top r - [\eps \norm{2}{\eta} - \l(\optImprove{\eps}\r)^\top \noisevec].$ From our assumption, $\frac{1}{\sigma^2}\norm{2}{\eta}^2$ is a chi-square random variable with degree $d,$ so from the Example 2.11 in~\citet{wainwright2019high}, we know with probability at least $1-\delta/2,$ one has
\begin{equation*}
    \frac{\norm{2}{\eta}^2}{d \sigma^2} \leq 1 + \sqrt{\frac{8 \log\l(2/\delta\r)}{d}}.
\end{equation*}
This implies
\begin{equation*}
    \norm{2}{\eta} 
    \leq \sqrt{d \sigma^2 \l(1 + \sqrt{\frac{8 \log\l(2/\delta\r)}{d}}\r)}
    \leq \sqrt{d} \sigma \l(1 + \sqrt{\frac{2 \log\l(2/\delta\r)}{d}}\r).
\end{equation*}
The last inequality comes from $\sqrt{1+u} \leq 1 + \frac{u}{2}$ for positive $u$. Moreover, since $\optImprove{\eps}$ is a fixed vector, we know $\l(\optImprove{\eps}\r)^\top \noisevec \sim \normal\l(0,\sigma^2 \norm{2}{\optImprove{\eps}}^2\r).$ So with probability at least $1-\delta/2,$ one has
\begin{equation*}
    \l(\optImprove{\eps}\r)^\top \noisevec \geq -\sigma \norm{2}{\optImprove{\eps}} \sqrt{2 \log\l(\frac{2}{\delta}\r)}
    \geq -\sigma \eps \sqrt{2 \log\l(\frac{2}{\delta}\r)}
\end{equation*}
Combining the two terms above, one has with probability at least $1-\delta,$ it holds
\begin{equation}\label{eqn.bound.noise}
    \eps \norm{2}{\eta} - \l(\optImprove{\eps}\r)^\top \noisevec \leq \eps \sqrt{d} \sigma \l(1 + \sqrt{\frac{2 \log\l(2/\delta\r)}{d}}\r) + \sigma \eps \sqrt{2 \log\l(\frac{2}{\delta}\r)}
    = \eps \sqrt{d} \sigma \l(1 + \sqrt{\frac{8 \log\l(2/\delta\r)}{d}}\r).
\end{equation}

\ 

Combining \eqref{eqn.signal.noise.decom}, \eqref{eqn.bound.signal} and \eqref{eqn.bound.noise}, we finish the proof.
\end{proof} 

\section{Monte Carlo computation}\label{sec:monte.carlo}

\begin{algorithm}[h]
   \caption{Monte-Carlo method for computing $\GaussSupremum{\eps}$}
   \label{alg:monte.carlo}
\begin{algorithmic}
    \STATE {\bfseries Input:} Offline dataset $\mathcal{D},$ the radius value $\eps \in E,$ the total sample size $M$ and threshold $M_0.$
    \STATE 1. Independently sample $M$ noise vectors, denoted as $\eta_i$ for $i \in [M],$ where $\eta_i \sim \normal0,\sigma_i^2/N(a_i), \sigma_i^2$ is the noise variance for the $i$-th arm and $N(a_i)$ is the sample size for $a_i$ in $\cD.$
    \STATE 2. Solve $X_i := \sup_{\improve \in \trustRegion{\eps}} \improve^\top \eta_i$ for $i \in [M]$ and order them as $X_{(1)} \leq X_{(2)} \leq ... \leq X_{(M)}.$
    \STATE 3. \textbf{Return} $X_{(M-M_0+1)}$ as an estimate of $\GaussSupremum{\eps}$ defined in \cref{def.G}.
\end{algorithmic} 
\end{algorithm}

As discussed in \cref{sec.TRUST}, we can estimate $\GaussSupremum{\eps}$ using classical Monte Carlo method. In this section, we illustrate the detailed implementation. We first sample $M$ i.i.d. noise and then solve $\sup_{\improve \in \trustRegion{\eps}}\improve^\top \noisevec$ for each to get $M$ suprema. We eventually select the $M_0$-th largest values of all suprema as our estimate for the bonus function, where $M_0$ is a pre-computed integer dependent on $M$ and the pre-determined failure probability $\delta > 0.$ Here, the program $\sup_{\improve \in \trustRegion{\eps}}\improve^\top \noisevec$ is a second-order cone program and can be efficiently solved via standard off-the shelf libraries~\citep{alizadeh2003second, boyd2004convex, diamond2016cvxpy}. The pseudocode for the Monte-Carlo sampling is in \cref{alg:monte.carlo}.

To determine $M_0,$ we denote $\eta_i$ as the i.i.d. noise vector for $i \in [M]$ and $X_i = \sup_{\improve \in \trustRegion{\eps}} \improve^\top \eta.$ We denote the order statistics of $X_i$-s as $X_{(1)} \leq X_{(2)} \leq ... \leq X_{(M)}.$ Suppose the cumulative distribution function of $X_i$ is $F(x),$ then from the property of the order statistics, we know the cumulative distribution function of $X_{(M-M_0+1)}$ is
\begin{equation*}
    F_{X_{\left(M-M_0+1\right)}}(x)=\sum_{j=M-M_0+1}^M C_M^j\left(F(x)\right)^j\left(1-F(x)\right)^{M-j}.
\end{equation*}
We denote $q_{1-\delta}$ as the $(1-\delta)$-lower quantile of the random variable $X$, then we have $F_{X_{\left(M-M_0+1\right)}}\left(q_{1-\delta}\right)=\sum_{j=M-M_0+1}^M C_M^j(1-\delta)^j(\delta)^{M-j}.$ For integer $M$ and $\delta>0$, we define $Q(M, \delta)$ as the maximal integer $M_0$ such that $\sum_{j=M-M_0+1}^M C_M^j(1-\delta)^j(\delta)^{M-j} \leq \delta.$ With this definition, we take a fixed $M$ and a total failure tolerance $\delta$ for all $\varepsilon \in E$, then we take
$$
M_0=Q\left(M, \frac{\delta}{2|E|}\right)
$$
as the threshold number. Under this choice, for any $\varepsilon \in E$, with probability at least $1-\delta / 2|E|$, it holds $X_{\left(M-M_0+1\right)}>q_{1-\delta / 2|E|}.$ On the other hand, with probability $1-\delta / 2|E|$, it holds that $\sup _{\Delta \in \mathrm{C}(\varepsilon)} \Delta^{\top} \eta \leq q_{1-\delta / 2|E|}$
This implies
$$
\sup _{\Delta \in \mathrm{C}(\varepsilon)} \Delta^{\top} \eta \leq q_{1-\delta / 2|E|}<X_{\left(M-M_0+1\right)}
$$
with probability at least $1-\delta /|E|$. From a union bound, we know with probability at least $1-\delta$, the bound above holds for any $\varepsilon \in E.$

\section{A fine-grained analysis to the suboptimality}\label{appendix.fine.grain.analysis}
We have shown a problem-dependent upper bound for the suboptimality in \eqref{eqn.suboptimality.gap}. In this section, we will give a further upper bound for $\GaussSupremum{\eps}$ and hence, for the suboptimality. We have the following theorem. The proof is deferred to \cref{appendix.proof.fine.grain}.

\begin{theorem}\label{thm.upper.bound.no.G}
    For a policy $\pi$ (deterministic or stochastic), we denote its reward value as $V^\pi$. TRUST has the following properties.
\begin{enumerate}[leftmargin=*]
    \item We denote a comparator policy as a triple $(\eps,\improve,\pi)$ such that $\eps = \sum_{i=1}^d \frac{\sigma_i^2 \improve_i^2}{N_i}, \pi = \widehat{\mu} + \improve.$ We take the discrete candidate set $E$ defined in \eqref{eqn.discrete.set.E}.
    With probability at least $1-\delta,$ for any stochastic comparator policy $(\eps,\improve,\pi),$ the sub-optimality of the output policy of \cref{alg:protocol} can be upper bounded as
    \begin{align}\label{eqn.suboptimality.gap.fine.grain}
       &V^{\pi} - V^{\pi_{TRUST}} 
        \leq 2 \sqrt{2 \sum_{i=1}^d \frac{\alpha \improve_i^2 \sigma_i^2}{N_i} \log\l(\frac{2|E|}{\delta}\r)}
        + 2 \min\l\{\sqrt{\sum_{i=1}^d \frac{\alpha \improve_i^2 \sigma_i^2}{N_i}}, 4D \sqrt{\log_+\l(\frac{4ed \sum_{i=1}^d \frac{\alpha \improve_i^2 \sigma_i^2}{N_i}}{D^2}\r)}\r\}
    \end{align}
    where $D$ is defined as any quantity satisfying
    \begin{equation}\label{eqn.def.D}
        D \geq \sqrt{\max_{i \in [d]} \l[\frac{\sigma_i^2}{N_i} - \frac{2\sigma_i^2}{N}\r] + \frac{\sum_{j=1}^d N_j \sigma_j^2}{N^2}}.
    \end{equation}
    $\alpha$ is the decaying rate defined in \eqref{eqn.discrete.set.E}, $\log_+(a) = \max(1,\log(a)).$
    \item (Comparison with the optimal policy)
    We further assume $\sigma_i = 1$ for $i \in [d]$ and assume the offline dataset is generated from the policy $\mu(\cdot)$ with $\min_{i \in [d]}\mu(a_i) > 0.$ Without loss of generality we assume $a_1$ is the optimal arm and denote the optimal policy as $\pi_*$. We write
    \begin{equation}
        C^* := \frac{1}{\mu(a_1)}, \quad 
        C_{\min} := \frac{1}{\min_{i \in [d]} \mu(a_i)}.
    \end{equation}
    When $N \geq 8C_{\min} \log(d/\delta),$ with probability at least $1-2\delta$, one has
    \begin{align}\label{eqn.claim.fine.grain.1}
        V^{\pi_*} - V^{\pi_{TRUST}}
        &\lesssim \sqrt{\frac{C_{\min}}{N} \log_+\l(\frac{d C^*}{C_{\min}}\r)} + \sqrt{\frac{C^*}{N} \log\l(\frac{2|E|}{\delta}\r)}.
    \end{align}
    Specially, when $C_{\min} \simeq C^*,$ we have with probability at least $1-2\delta,$
    \begin{equation}\label{eqn.claim.fine.grain.2}
        V^{\pi_*} - V^{\pi_{TRUST}} 
        \lesssim \sqrt{\frac{C^*}{N} \log\l(\frac{2 d |E|}{\delta}\r)}.
    \end{equation}
    \end{enumerate}
\end{theorem}

We remark that \eqref{eqn.suboptimality.gap.fine.grain} is problem-dependent, and it gives an explicit upper bound for $\GaussSupremum{\lceil\eps\rceil}$ in \eqref{eqn.suboptimality.gap}. This is derived by first concentrating $\GaussSupremum{\eps}$ around $\E \sup_{\improve \in \trustRegion{\eps}} \improve^\top \eta$, which is well-defined as localized  Gaussian width or local Gaussian complexity~\citep{koltchinskii2006local}, and then upper bounding the localized Gaussian width of a convex hull via tools in convex analysis~\citep{bellec2019localized}. Different from \eqref{eqn.upper.bound.LCB.first}, when $\pi = a_i$ represents a single arm, \eqref{eqn.suboptimality.gap.fine.grain} relies not only on $\sigma_i^2 / N_i$, but on $\sigma_j^2 / N_j$ for $j \neq i$ as well, since the size of trust regions depend on $\sigma_i^2 / N_i$ for all $i \in [d].$

Notably, \eqref{eqn.claim.fine.grain.2} gives an analogous upper bound depending on $\mu(\cdot)$ and $N$, which is comparable to the bound for LCB in \eqref{eqn.upper.bound.LCB.second} up to constant and logarithmic factors. This indicates that, when behavioral cloning policy is not too imbalanced,  TRUST is guaranteed to achieve the same level of performance as LCB. In fact, this improvement is remarkable since TRUST is exploring a much larger policy searching space than LCB, which encompasses all stochastic policies (the whole simplex) rather than the set of all single arms only. We also remark that both \eqref{eqn.upper.bound.LCB.second} and \eqref{eqn.fine.grained2} are worst-case upper bound, and in practice, we will show in \cref{sec.simulated.experiment} that in some settings, TRUST can achieve good performance while LCB fails completely.

\paragraph{Is TRUST minimax-optimal?}
We consider the hard cases in MAB~\citep{rashidinejad2021bridging} where LCB achieves the minimax-optimal upper bound and we show for these hard cases, TRUST will achieve the same sample complexity as LCB up to log and constant factors. More specifically, we consider a two-arm MAB $\cA = \l\{1,2\r\}$ and the uniform behavioral cloning policy $\mu(1) = \mu(2) = 1/2.$ For $\delta \ in [0,1/4],$ we define $\mdp_1$ and $\mdp_2$ are two MDPs whose reward distributions are as follows.
\begin{align*}
    &\mdp_1: r(1) \sim \bernoulli{\frac{1}{2}},\ r(2) \sim \bernoulli{\frac{1}{2} + \delta} \\
    &\mdp_2: r(1) \sim \bernoulli{\frac{1}{2}},\ r(2) \sim \bernoulli{\frac{1}{2} - \delta},
\end{align*}
where $\bernoulli{p}$ is the Bernoulli distribution with probability $p.$ The next result is a corollary from \cref{thm.upper.bound.no.G}. 

\begin{corollary}
    We define $\mdp_1,mdp_2$ as above for $\delta \in [0,1/4].$ Assume $N \geq \widetilde O(1).$ Then, we have
    \begin{enumerate}[leftmargin=*]
        \item The minimax optimal lower bound for the suboptimality of LCB is
        \begin{equation}
            \inf_{\widehat a_\lcb \in \cA} \sup_{\mdp \in \l\{\mdp_1,\mdp_2\r\}} \E_\mathcal{D} \l[r(a^*) - r(\widehat a_\lcb) \r] \gtrsim \sqrt{\frac{C^*}{N}},
        \end{equation}
        where $\E_\mathcal{D}\l[\cdot\r]$ is the expectation over the offline dataset $\cD.$
        \item The upper bound for suboptimality of TRUST mathces the lower bound above up to log factor. Namely, for any $\mdp \in \l\{\mdp_1, \mdp_2\r\},$ one has
        \begin{equation}
            \E_\mathcal{D} \l[r(a^*) - V^{\pi_{TRUST}} \r] 
            \lesssim \sqrt{\frac{C^* \log(dN)}{N}}.
        \end{equation}
    \end{enumerate}
\end{corollary}

The first claim comes from Theorem 2 of~\citep{rashidinejad2021bridging}, while the second claim is a direct corollary to \cref{thm.upper.bound.no.G}.

\subsection{Proof of \cref{thm.upper.bound.no.G}}\label{appendix.proof.fine.grain}

\begin{proof}
    Recall from \cref{thm.performance.tr} that for any comparator policy $(\eps,\improve,\pi)$ defined above, one has
    \begin{equation*}
        V^{\pi} - V^{\pi_{TRUST}} 
        \leq 2\GaussSupremum{\lceil\eps\rceil},
    \end{equation*}
    where $\lceil\eps\rceil := \inf\l\{\eps^\prime \in E: \eps \leq \eps^\prime\r\}.$ The following lemma upper bounds the quantile of Gaussian suprema $\GaussSupremum{\eps}$ for each $\eps \in E.$ The proof is deferred to \cref{appendix.proof.upper.bound.G}.
    
    \begin{lemma}\label{thm.upper.bound.G}
        For $\eps \in E,$ one can upper bound $\GaussSupremum{\eps}$ as follows.
        \begin{equation}\label{eqn.upper.bound.G}
            \GaussSupremum{\eps} \leq \min\l\{\eps \cdot \sqrt{d} \ , \ 4D \sqrt{\log_+\l(\frac{4ed\eps^2}{D^2}\r)}\r\} + \sqrt{2\eps^2 \log\l(\frac{2|E|}{\delta}\r)}
        \end{equation}
        where $\log_+(a) = \max(1,\log(a))$ and $D$ is a quantity satisfying
        \begin{equation}\label{eqn.def.D}
            D \geq \sqrt{\max_{i \in [d]} \l[\frac{\sigma_i^2}{N_i} - \frac{2\sigma_i^2}{N}\r] + \frac{\sum_{j=1}^d N_j \sigma_j^2}{N^2}}.
        \end{equation}
    \end{lemma}
    Applying \cref{thm.upper.bound.G} to $\lceil \eps \rceil \in E,$ we obtain
    \begin{align}\label{eqn.fine.grain.upper.bound1}
        V^{\pi} - V^{\pi_{TRUST}} 
        \leq 2\min\l\{\lceil\eps\rceil \cdot \sqrt{d} \ , \ 4D \sqrt{\log_+\l(\frac{4ed \lceil\eps\rceil ^2}{D^2}\r)}\r\} 
        + 2\sqrt{2\lceil\eps\rceil^2 \log\l(\frac{2|E|}{\delta}\r)}.
    \end{align}
    Since $\eps = \sum_{i=1}^d \frac{\sigma_i^2 \improve_i^2}{N_i}$, we know from our discretization scheme in \eqref{eqn.discrete.set.E}
    \begin{equation}\label{eqn.fine.grain.upper.bound2}
        \lceil \eps \rceil \leq \alpha \cdot \sum_{i=1}^d \frac{\sigma_i^2 \improve_i^2}{N_i}.
    \end{equation}
    Bridging \eqref{eqn.fine.grain.upper.bound2} into \eqref{eqn.fine.grain.upper.bound1}, we obtain our first claim. In order to get the second claim, we take $\sigma_i = 1$ for $i \in [d]$ and $\improve = \pi_* - \widehat \mu,$ which is the vector pointing at the vertex corresponding to the optimal arm from the uniform reference policy $\widehat \mu$ defined in \eqref{eqn.uniform.general}. Then, we have
    \begin{equation*}
        \sum_{i=1}^d \frac{\improve_i^2 \sigma_i^2}{N_i} = \frac{1}{N_1} - \frac{2}{N} + \frac{1}{N}
        = \frac{1}{N_1} - \frac{1}{N}
        \leq \frac{1}{N_1},
    \end{equation*}
    where $N_1$ is the sample size for the optimal arm $a_1.$ Therefore, we can further bound \eqref{eqn.fine.grain.upper.bound1} as 
    \begin{align}\label{eqn.fine.grain.upper.bound3}
        V^{\pi_*} - V^{\pi_{TRUST}} 
        &\leq 4D \sqrt{\log_+ \l(\frac{4\alpha ed}{N_1 D^2}\r)}
        + 2 \sqrt{\frac{2\alpha}{N_1} \log\l(\frac{2|E|}{\delta}\r)}.
    \end{align}
    Finally, we take a specific value of $D$ and lower bound $N_1$ via Chernoff bound in \cref{lem.chernoff.MAB}. From \cref{lem.chernoff.MAB}, we know that when $N \geq 8C_{\min} \log(d/\delta),$ with probability at least $1-\delta,$ we have
    \begin{equation}\label{eqn.fine.grain.event}
        N_i \geq \frac{1}{2} N \mu(a_i) 
    \end{equation}
    for any $i \in [d].$ Recall the definition of $D$ in \eqref{eqn.def.D}, we know that $D$ can be arbitrary value greater than $\sqrt{\max_{i \in [d]} \l[\frac{\sigma_i^2}{N_i} - \frac{2\sigma_i^2}{N}\r] + \frac{\sum_{j=1}^d N_j \sigma_j^2}{N}}.$ Then, when $\sigma_i = 1$, one has
    \begin{equation*}
        \sqrt{\max_{i \in [d]} \l[\frac{\sigma_i^2}{N_i} - \frac{2\sigma_i^2}{N}\r] + \frac{\sum_{j=1}^d N_j \sigma_j^2}{N^2}} \leq \sqrt{\frac{1}{\min_{i \in [d]} N_i}}.
    \end{equation*}
    We denote $N_j = \min_{i \in [d]} N_i$ (when there are multiple minimizers, we arbitrarily pick one). Then, we have
    \begin{equation*}
        \sqrt{\max_{i \in [d]} \l[\frac{\sigma_i^2}{N_i} - \frac{2\sigma_i^2}{N}\r] + \frac{\sum_{j=1}^d N_j \sigma_j^2}{N^2}} 
        \leq \sqrt{\frac{1}{N_j}}
        \leq \sqrt{\frac{2}{N \mu(a_j)}}
        \leq \sqrt{\frac{2}{N \cdot \min_{i \in [d]} \mu(a_i)}}
        = \sqrt{\frac{2C_{\min}}{N}}.
    \end{equation*}
    Therefore, we take $D = \sqrt{\frac{2C_{\min}}{N}}$ in \eqref{eqn.fine.grain.upper.bound3} and apply $N_1 \geq \frac{1}{2} N \mu(a_i)$ to obtain
    \begin{equation}
        V^{\pi_*} - V^{\pi_{TRUST}} 
        \leq 4 \sqrt{\frac{2 C_{\min}}{N} \log_+\l(\frac{4\alpha e d C^*}{C_{\min}}\r)} + 4 \sqrt{\frac{\alpha C^*}{N} \log\l(\frac{2|E|}{\delta}\r)},
    \end{equation}
    which proves \eqref{eqn.claim.fine.grain.1}. Finally, when $C^* \simeq C_{\min},$ one has 
    \begin{equation*}
        V^{\pi_*} - V^{\pi_{TRUST}} 
        \lesssim \sqrt{\frac{C^*}{N} \log\l(\frac{2 d |E|}{\delta}\r)}.
    \end{equation*}
    Therefore, we conclude.
\end{proof}

\subsection{Proof of \cref{thm.upper.bound.G}}\label{appendix.proof.upper.bound.G}

\begin{proof}
    Recall that $\improve = (\improve_1, \improve_2,...,\improve_d)^\top$ is the improvement vector, $\eta = (\eta_1,\eta_2,...,\eta_d)^\top$ is the noise vector, where entries are independent and $\eta_i \sim \normal(0,\sigma_i^2/N_i)$ and $N_i$ is the sample size of arm $a_i$ in the offline dataset. To proceed with the proof, let's further define
    \begin{equation}
        \widetilde\eta = \l(\widetilde\eta_1, \widetilde\eta_2,...,\widetilde\eta_d\r)^\top, \quad
        \widetilde\improve = \l(\widetilde\improve_1, \widetilde\improve_2,...,\widetilde\improve_d\r)^\top,
        \quad \text{where}\quad
        \widetilde\eta_i = \eta_i \frac{\sqrt{N_i}}{\sigma_i}, \ 
        \widetilde\improve_i = \frac{\improve_i \sigma_i}{\sqrt{N_i}}.
    \end{equation}
    With this notation, one has 
    \begin{equation*}
        \widetilde\eta \sim \normal\l(0,I_d\r), \quad 
        \eta^\top \improve = \widetilde\eta^\top \widetilde\improve.
    \end{equation*}
    We also write the equivalent trust region (for $\widetilde\improve$) as
    \begin{equation}
        \tTrustRegion{\eps} = \l\{\widetilde\improve \in \R^d: \frac{\sqrt{N_i}}{\sigma_i} \widetilde\improve_i + \widehat\mu_i \geq 0, \quad 
        \sum_{i=1}^d\l[\frac{\sqrt{N_i}}{\sigma_i} \widetilde\improve_i + \widehat\mu_i\r]=1,\quad
        \norm{2}{\widetilde\improve} \leq \eps\r\},
    \end{equation}
    where $\widehat\mu = (\widehat\mu_1,\widehat\mu_2,...,\widehat\mu_d)^\top$ is the policy weight for the reference policy. From the definition above, one has for any $\eps > 0,$
    \begin{equation*}
        \improve \in \trustRegion{\eps} \ \Leftrightarrow \ \timprove \in \tTrustRegion{\eps}.
    \end{equation*}
    Then, we apply \cref{lem.concentration.gaussian.suprema} to $\sup_{\improve \in \trustRegion{\eps}} \improve^\top \eta$ for a $\eps \in E.$ One has with probability at least $1-\frac{\delta}{|E|},$
    \begin{align*}
        \l|\sup_{\improve \in \trustRegion{\eps}} \improve^\top \eta - \E \sup_{\improve \in \trustRegion{\eps}} \improve^\top \eta\r| \leq \sqrt{2\eps^2 \log\l(\frac{2|E|}{\delta}\r)}.
    \end{align*}
    From a union bound, one immediately has with probability at least $1-\delta,$ for any $\eps \in E,$ it holds that
    \begin{equation}\label{eqn.fine.grained1}
        \sup_{\improve \in \trustRegion{\eps}} \improve^\top \eta \leq \E \sup_{\improve \in \trustRegion{\eps}} \improve^\top \eta + \sqrt{2\eps^2 \log\l(\frac{2|E|}{\delta}\r)}.
    \end{equation}
    From the definition of $\GaussSupremum{\eps}$ in \eqref{def.G}, we know that $\GaussSupremum{\eps}$ is the minimal quantity that satisfy \eqref{eqn.fine.grained1} with probability at least $1-\delta.$ Therefore, one has
    \begin{equation}\label{eqn.fine.grained2}
        \GaussSupremum{\eps} 
        \leq \E \sup_{\improve \in \trustRegion{\eps}} \improve^\top \eta + \sqrt{2\eps^2 \log\l(\frac{2|E|}{\delta}\r)} 
        = \E_{\teta \sim \normal(0,I_d)} \l[\sup_{\timprove \in \tTrustRegion{\eps}} \timprove^\top \teta\r] + \sqrt{2\eps^2 \log\l(\frac{2|E|}{\delta}\r)}
        \quad \forall \eps \in E.
    \end{equation}
    Note that, the first term in the RHS of \eqref{eqn.fine.grained2} is well-defined as localized Gaussian width over the convex hull defined by the trust region $\trustRegion{\eps}$ (or equivalently, $\tTrustRegion{\eps}$). We denote 
    \begin{equation}\label{eqn.def.convex.hull}
        T := \l\{\widetilde\improve \in \R^d: \frac{\sqrt{N_i}}{\sigma_i} \widetilde\improve_i + \widehat\mu_i \geq 0, \quad 
        \sum_{i=1}^d\l[\frac{\sqrt{N_i}}{\sigma_i} \widetilde\improve_i + \widehat\mu_i\r]=1\r\}.
    \end{equation}
    We immediately have that $T$ is a convex hull of $d$ points in $\R^d$ and the vertices of this convex hull are the vertices of the simplex in $\R^d$ shifted by the reference policy $\widehat \mu.$ In what follows, we plan to apply \cref{lem.localized.Gaussian.width} to the localized Gaussian width of $T \cap \eps \ball_2.$ However, $T$ is not subsumed by the unit ball in $\R^d,$ so we need to do some additional scaling. Note that, the zero vector is included in $T.$ Let's compute the farthest distance for the vertices of $T.$ We denote the $i$-th vertex of $T$ as
    \begin{equation}
        \timprove = \l(-\frac{\sigma_1}{\sqrt{N_1}} \widehat\mu_1,..., -\frac{\sigma_{i-1}}{\sqrt{N_{i-1}}} \widehat\mu_{i-1},
        \frac{\sigma_i}{\sqrt{N_i}} \l(1-\widehat \mu_i\r),
        -\frac{\sigma_{i+1}}{\sqrt{N_{i+1}}} \widehat\mu_{i+1},..., -\frac{\sigma_{d}}{\sqrt{N_{d}}} \widehat\mu_{d}\r).
    \end{equation}
    The $\ell_2$-norm of this improvement vector is 
    \begin{equation*}
        \norm{2}{\timprove} = \sqrt{\frac{\sigma_i^2}{N_i} - \frac{2\sigma_i^2}{N} + \frac{\sum_{i=1}^d N_i \sigma_i^2}{N^2}},
    \end{equation*}
    where $N$ is the total sample size of the offline dataset. Therefore, the maximal radius of $T$ can be upper bounded by $D$, where $D$ is any quantity that satisfies
    \begin{equation}
        D \geq \sqrt{\max_{i \in [d]} \l[\frac{\sigma_i^2}{N_i} - \frac{2\sigma_i^2}{N}\r] + \frac{\sum_{j=1}^d N_j \sigma_j^2}{N^2}}.
    \end{equation}
    We denote $S = \frac{1}{D} \cdot T := \l\{\frac{1}{D} \cdot x: x \in T\r\}.$ Then, from \cref{lem.localized.Gaussian.width}, one has
    \begin{align}
        \E_{\teta \sim \normal(0,I_d)} \l[\sup_{\timprove \in \tTrustRegion{\eps}} \timprove^\top \teta\r]
        &=\E_{\teta \sim \normal(0,I_d)} \l[\sup_{\timprove \in T \cap \eps \ball_2} \timprove^\top \teta\r] \notag \\
        &= D \cdot \E_{\teta \sim \normal(0,I_d)} \l[\sup_{\timprove \in S \cap \frac{\eps}{D}\cdot \ball_2} \timprove^\top \teta\r] \tag{$S \cap \frac{\eps}{D}\cdot \ball_2$ can be got by scaling $T \cap \eps \ball_2$ bt $\frac{1}{D}$} \\
        &\leq D \cdot \l[\l(4 \sqrt{\log_+\l(4ed\l(\frac{\eps^2}{D^2}\wedge 1\r)\r)}\r) \wedge \l(\frac{\eps}{D} \sqrt{d}\r)\r] \tag{Take $s = \frac{\eps}{D}$ and $M=d$ in \cref{lem.localized.Gaussian.width}} \\
        &= D \cdot \l[\l(4 \sqrt{\log_+\l(4ed\l(\frac{\eps^2}{D^2}\r)\r)}\r) \wedge \l(\frac{\eps}{D} \sqrt{d}\r)\r] \tag{$\eps \leq D$ for any $\eps \in E$}.
    \end{align}
    This finishes the proof.
\end{proof}

\subsection{Auxiliary lemmas}
\begin{lemma}[Concentration of Gaussian suprema, Exercise 5.10 in~\citet{wainwright2019high}]\label{lem.concentration.gaussian.suprema}
    Let $\left\{X_\theta, \theta \in \mathbb{T}\right\}$ be a zero-mean Gaussian process, and define $Z=\sup _{\theta \in \mathbb{T}} X_\theta$. Then, we have
    $$
    \mathbb{P}[|Z-\mathbb{E}[Z]| \geq \delta] \leq 2 \exp\l(-\frac{\delta^2}{2 \sigma^2}\r),
    $$
    where $\sigma^2:=\sup _{\theta \in \mathbb{T}} \operatorname{var}\left(X_\theta\right)$ is the maximal variance of the process.
\end{lemma}

\ 

\begin{lemma}[Localized Gaussian Width of a Convex Hull, Proposition 1 in~\citet{bellec2019localized}]\label{lem.localized.Gaussian.width}
    Let $d \geq 1, M \geq 2$ and $T$ be the convex hull of $M$ points in $\R^d.$ We write $\ball_2 = \l\{x \in \R^d: \norm{2}{x} \leq 1\r\}$ and $s \ball_2 = \l\{s \cdot x: x \in \R^d, \norm{2}{x} \leq 1\r\}.$ Assume $T \subset \ball_2^d(1).$ Let $g \in \R^d$ be a standard Gaussian vector. Then, for all $s > 0,$ one has
    \begin{equation}
        \E \l[\sup_{x \in T \cap s \ball_2} \ x^\top g\r]
        \leq \l(4 \sqrt{\log_+\l(4eM\l(s^2 \wedge 1\r)\r)}\r) \wedge \l(s \sqrt{d \wedge M}\r),
    \end{equation}
    where $\log_+(a) = \max(1,\log(a)), a \wedge b = \min\l\{a,b\r\}.$ 
\end{lemma}

\ 

\begin{lemma}[Chernoff bound for binomial random variables, Theorem 2.3.1 in~\citet{vershynin2020high}]\label{lem.chernoff.bound}
    Let $X_i$ be independent Bernoulli random variables with parameters $p_i$. Consider their sum $S_N=\sum_{i=1}^N X_i$ and denote its mean by $\mu=\mathbb{E} S_N$. Then, for any $t>\mu$, we have
    $$
    \mathbb{P}\left\{S_N \geq t\right\} \leq e^{-\mu}\left(\frac{e \mu}{t}\right)^t.
    $$
\end{lemma}

\ 

\begin{lemma}[Chernoff bound for offline MAB]\label{lem.chernoff.MAB}
    Under the setting in \cref{thm.upper.bound.no.G}, we have 
    \begin{equation*}
        \P \l(N_i \geq \frac{1}{2} N \mu(a_i) \quad \forall  i \in [d]\r) \leq 1 - d\exp\l(-\frac{N \cdot \min_{j \in [d]} \mu(a_j)}{8}\r),
    \end{equation*}
\end{lemma}
\begin{proof}
    For arm $i \in [d],$ we take $\mu = N\mu(a_i)$ and $t = \frac{1}{2} M \mu(a_i)$ in \cref{lem.chernoff.bound} and obtain
    \begin{equation*}
        \P \l(N_i \geq \frac{1}{2} N \mu(a_i)\r) \leq \exp\l(-N \mu(a_i)\r) \cdot \l(\frac{e N \mu(a_i)}{\frac{1}{2} N \mu(a_i)}\r)^{\frac{1}{2}N \mu(a_i)}
        = \exp \l(N \mu(a_i) \l[-1 + \frac{1}{2} \log(2e)\r]\r)
        \leq \exp\l(-\frac{N\mu(a_i)}{8}\r).
    \end{equation*}
    We finish the proof by a union bound for all arms. 
\end{proof}

\section{Proof of \cref{lem.equivalent.form}}
\begin{proof}
Recall the definition of $\lceil \eps \rceil:$
\begin{equation}\label{eqn.def.ceil.eps}
    \lceil \eps \rceil := \inf\l\{\eps^\prime \in E: \eps^\prime \geq \eps\r\}.
\end{equation}
We additionally define
\begin{equation}\label{eqn.def.floor.eps}
    \lfloor \eps \rfloor := \sup\l\{\eps^\prime \in E: \eps^\prime < \eps\r\}.
\end{equation}
Specially, if there is no $\eps^\prime \in E$ such that $\eps^\prime < \eps,$ then we define $\lfloor \eps \rfloor = 0$. Then we know for any $\eps \leq \eps_0 \in E$ ($\eps_0$ is the largest possible radius) and a finite set $E,$ it holds that
\begin{equation}
    \lfloor \eps \rfloor < \eps \leq \lceil \eps \rceil, \quad \text{ and } \quad \eps = \lceil \eps \rceil \text{ if and only if } \eps \in E.
\end{equation}
For any $\eps \in E$, recall $\empImprove{\eps}$ is the optimal improvement vector within $\trustRegion{\eps}$ defined in \eqref{eqn.optimal.epss.per.stage}. It holds that
\begin{align}
    \empImprove{\eps} :
    &= \mathop{\arg\max}_{\improve \in \trustRegion{\eps}} \improve^\top \empR 
    = \mathop{\arg\max}_{\improve \in \trustRegion{\eps}} \big[\improve^\top \empR - \GaussSupremum{\eps}\big] \tag{since $\GaussSupremum{\eps}$ does not depend on $\improve$} \\
    &= \mathop{\arg\max}_{\improve \in \trustRegion{\eps}} \big[\improve^\top \empR - \GaussSupremum{\lceil\eps\rceil}\big] \tag{$\eps \in E,$ so $\lceil \eps \rceil = \eps$} \\
    &\leq \mathop{\arg\max}_{\eps^\prime \in \l(\lfloor \eps\rfloor, \lceil \eps \rceil\r], \improve \in \trustRegion{\eps^\prime}} \big[\improve^\top \empR - \GaussSupremum{\lceil\eps^\prime\rceil} \big]. \notag
\end{align}
On the other hand, when $\eps \in E$ and $\eps^\prime \in (\floor{\eps},\ceil{\eps}],$ one has $\ceil{\eps^\prime} = \ceil{\eps} = \eps,$ so
\begin{align*}
    \mathop{\arg\max}_{\eps^\prime \in \l(\lfloor \eps\rfloor, \lceil \eps \rceil\r], \improve \in \trustRegion{\eps^\prime}} \big[\improve^\top \empR - \GaussSupremum{\lceil\eps^\prime\rceil}  \big] 
    &= \mathop{\arg\max}_{\eps^\prime \in \l(\lfloor \eps\rfloor, \lceil \eps \rceil\r], \improve \in \trustRegion{\eps^\prime}} \big[\improve^\top \empR - \GaussSupremum{\lceil\eps\rceil} \big]
    \leq \mathop{\arg\max}_{\improve \in \trustRegion{\eps}} \big[\improve^\top \empR - \GaussSupremum{\lceil\eps\rceil} \big],
\end{align*}
where the last inequality comes from the fact that $\trustRegion{\eps^\prime} \subset \trustRegion{\eps}$ when $\eps^{\prime} \leq \lceil \eps \rceil = \eps$ by definition of the trust region in \eqref{eqn.def.trust.region}.
Combining two inequalities above, we have for any $\eps \in E,$
\begin{equation}
    \l(\eps, \widehat \improve_{\eps}\r) = \mathop{\arg\max}_{\eps^\prime \in \l(\lfloor \eps\rfloor, \lceil \eps \rceil\r], \improve \in \trustRegion{\eps^\prime}} \big[\improve^\top \empR - \GaussSupremum{\lceil\eps^\prime \rceil} \big],
\end{equation}
where the variables in RHS above are $\eps^\prime$ and $\improve$, and 
Therefore, from the definition of we have
\begin{align*}
    \l(\widehat \eps_*, \selectImprove\r) = \mathop{\arg\max}_{\eps \in E} \mathop{\arg\max}_{\eps^\prime \in \l(\lfloor \eps\rfloor, \lceil \eps \rceil\r], \improve \in \trustRegion{\eps^\prime}} \big[\improve^\top \empR - \GaussSupremum{\lceil\eps^\prime \rceil} \big]
    = \mathop{\arg\max}_{\eps \leq \eps_0, \improve \in \trustRegion{\eps}} \big[\improve^\top \empR - \GaussSupremum{\lceil\eps \rceil} \big].
\end{align*}
This finishes the proof.
\end{proof}

\section{Augmentation with LCB}\label{sec.combination.LCB}
To determine the most effective final policy, we can compare the outputs of the LCB and \cref{alg:protocol} and combine both policies, based on the relative magnitude of their corresponding lower bounds. Specifically, the combined policy is
\begin{align}\label{eqn.def.combined.policy}
    &\pi_\mathsf{combined} = \notag \\
    &\left\{
    \begin{aligned}
        \widehat{a}_\lcb & \ \text{ If } \max_{a_i \in \cA} l_i \geq w_{\mathsf{TR}}^\top \widehat{r} - \GaussSupremum{\lceil \widehat \eps_* \rceil} - \sqrt{\frac{2\log(1/\delta)}{\sum_{j=1}^d N_j / \sigma_j^2}}, \\
        w_{\mathsf{TR}} & \ \text{ If } \max_{a_i \in \cA} l_i < w_{\mathsf{TR}}^\top \widehat{r} - \GaussSupremum{\lceil \widehat \eps_* \rceil} - \sqrt{\frac{2\log(1/\delta)}{\sum_{j=1}^d N_j / \sigma_j^2}},
    \end{aligned}
    \right.
\end{align}
where $l_i = \widehat r_i - b_i$ is defined in \eqref{eqn.emp.reward.mab} and $\GaussSupremum{\eps}$ is defined in \cref{def.G}. This combined policy will perform at least as well as LCB with high probability. More specifically, we have
\begin{corollary}
    We denote the arm chosen by LCB as $\widehat a_\lcb$. We also denote $r(\cdot)$ as the true reward of a policy (deterministic or stochastic). With probability at least $1-3\delta,$ one has
    \begin{equation}
        V^{\pi_{\mathsf{combined}}} \geq \max_{a_i \in \cA} l_i.
    \end{equation}
\end{corollary}

\begin{proof}
    We denote $\widehat r\l(\widehat a_\lcb\r) = r_{\widehat a_\lcb}$ and $\widehat r(w_{\mathsf{TR}})$ as the empirical reward of the policy returned by LCB and \cref{alg:protocol}, respectively. Recall the uncertainty term of LCB in \eqref{eqn.emp.reward.mab} and of \cref{alg:protocol} in \eqref{eqn.def.combined.policy}, we write $b(\widehat a_\lcb) = b_{\widehat a_\lcb}$ and $b(w_{\mathsf{TR}}) = \GaussSupremum{\lceil \widehat \eps_* \rceil} + \sqrt{2\log(1/\delta) / [\sum_{j=1}^d N_j / \sigma_j^2]}$. Then, from \cref{thm.performance.tr}, \eqref{eqn.confidence.bound.LCB} and a union bound, we know with probability at least $1-3\delta,$ it holds that
    \begin{equation*}
        r(\widehat a_\lcb) \geq \widehat r(\widehat a_\lcb) - b(\widehat a_\lcb), \
        r(w_{\mathsf{TR}}) \geq \widehat r(w_{\mathsf{TR}}) - b(w_{\mathsf{TR}}),
    \end{equation*}
    which implies 
    \begin{align*}
        V^{\pi_{\mathsf{combined}}}
        &\geq \widehat r(\pi_{\mathsf{combined}}) - b(\pi_{\mathsf{combined}}) \\
        &\geq \widehat r(\widehat a_\lcb) - b(\widehat a_\lcb) \tag{By \eqref{eqn.def.combined.policy}} \\
        &= \max_{a_i \in \cA} l_i. \tag{By the definition of $\widehat a_\lcb$ in \eqref{eqn.lcb.algorithm}}
    \end{align*}
    Therefore, we conclude.
\end{proof}

\section{Experiment details}\label{sec.experiment.details}

We did experiments on Mujoco environment in the D4RL dataset~\citep{fu2020d4rl}. All environments we test on are v3. Since the original D4RL dataset does not include the exact form of logging policies, we retrain SAC~\citep{haarnoja2018soft} on these environment for 1000 episodes and keep record of the policy in each episode. We test 4 environments in two settings, denoted as '1-traj-low' and '1-traj-high'. In either setting, the offline dataset is generated from 100 policies with one trajectory from each. In the '1-traj-low' setting, the data is generated from the first 100 policies in the training process of SAC, while in the '1-traj-high' setting, it is generated from the policy in $(10x+1)$-th episodes in the training process. 

For all experiments on Mujoco, we average the results over 4 random seeds (from 2023 to 2026), and to run CQL, we use default hyper-parameters in \url{https://github.com/young-geng/CQL} to run 2000 episodes. For TRUST, we run it using a fixed standard deviation level $\sigma_i = 150$ for all experiments.

\end{document}